\documentclass[11pt]{article}
\usepackage[pdftex]{graphicx}
\usepackage{tikz}
\usepackage{fullpage}
\usepackage[T1]{fontenc}
 \usepackage[parfill]{parskip}
\usepackage{lipsum}

\usepackage{authblk}
\usepackage{amssymb,amsthm,amsmath,amssymb,wrapfig,dsfont}
\usepackage{thmtools,thm-restate}
\PassOptionsToPackage{hyphens}{url}
\usepackage{varioref}
\usepackage{verbatim}
\definecolor{DarkBlue}{rgb}{0.1,0.1,0.5}
\definecolor{DarkGreen}{rgb}{0.1,0.5,0.1}
\usepackage[backref=page]{hyperref}
\hypersetup{
     colorlinks   = true,
     linkcolor    = DarkBlue, 
     urlcolor     = DarkBlue, 
	 citecolor    = DarkBlue 
}
\renewcommand*{\backref}[1]{}
\renewcommand*{\backrefalt}[4]{%
    \ifcase #1 (Not cited.)%
    \or        (Cited on page~#2)%
    \else      (Cited on pages~#2)%
    \fi}
\usepackage[capitalise,noabbrev]{cleveref}

\usepackage{verbatim}
\usepackage{microtype}
\usepackage{subfigure}
\usepackage{enumitem}
\usepackage{booktabs} 
\usepackage{natbib}
\usepackage{color}
\usepackage{comment}
\usepackage{thmtools, thm-restate}
\usepackage{hyperref,xcolor}
\makeatletter
\newcommand{\printfnsymbol}[1]{%
  \textsuperscript{\@fnsymbol{#1}}%
}
\makeatother
\usepackage{wrapfig}
\usepackage{verbatim}
\usepackage{microtype}
\usepackage{subfigure}
\usepackage{enumitem}
\usepackage{booktabs} 
\usepackage{color,graphicx}
\usepackage{comment}
\usepackage{thmtools, thm-restate}
\usepackage{hyperref,xcolor}
\newcommand{\err}{{\mathtt{err}}}
\newcommand{\fstar}{{f^{\star}}}
\newcommand{\fhat}{{\widehat{f}}}
\newcommand{\pr}{{\mathbb{P}}}
\usepackage{hyperref}
\usepackage{amssymb,amsmath,amsthm,dsfont,bm,mathtools}
\newcommand{\red}[1]{{\leavevmode\color{red}{#1}}}

\newcommand\todo[1]{{\red{TODO: {#1}}}}

\newcommand\toref{\red{[REF]}}

\newcommand{\vineet}[1]{{#1}}
\newcommand{\ganesh}[1]{{#1}}
\newcommand{\naive}{{na\"ive}}

\newcommand\expect[2]{\mathbb{E}_{#1}{\left[ {#2} \right]}}

\DeclareMathOperator*{\argmax}{argmax}
\DeclareMathOperator*{\argmin}{argmin}
\DeclareMathOperator{\sign}{sign}
\newcommand{\1}[1]{\mathds{1}{\{{#1}\}}}

\newtheorem{theorem}{Theorem}

\newcommand{\X}{{\cal{X}}}
\newcommand{\calH}{{\cal{H}}}

\newcommand{\Y}{{\cal{Y}}}
\newcommand{\R}{{\mathbb{R}}}

\newcommand{\fapx}{{\hat{f}}}
\newcommand{\Fapx}{{\hat{F}}}

\newcommand{\iid}{{\stackrel{_{\text{iid}}}{\sim}}}
\newcommand{\cg}{{\texttt{cg}}}

\newtheorem{observation}[theorem]{Observation}
\newtheorem{definition}{Definition}

\newtheorem{lemma}[theorem]{Lemma}
\newtheorem{claim}[theorem]{Claim}

\newtheorem{corollary}[theorem]{Corollary}
\newcommand{\wvec}{{\mathbf{w}}}
\newcommand{\xvec}{{\mathbf{x}}}
\newcommand{\yvec}{{\mathbf{y}}}
\newcommand{\admissible}{admissible}
\newcommand{\SVC}{\textnormal{SVC}}
\newcommand{\VC}{\textnormal{VC}}
\newcommand{\POP}{\textsc{POP}}

\title{Strategic Classification in the Dark
}

\author{Ganesh Ghalme\footnote{Equal contribution, alphabetical order.}}
\author{Vineet Nair\printfnsymbol{1}}
\author{Itay Eilat}
\author{Inbal Talgam-Cohen}
\author{Nir Rosenfeld}
\affil{Technion Israel Institute of Technology, Haifa, Israel}

\begin{document}
\maketitle

\begin{abstract}

Strategic classification studies the interaction between a classification rule and the strategic agents it governs. Under the assumption that the classifier is \emph{known}, rational agents respond to it by manipulating their features. However, in many real-life scenarios of high-stake classification (e.g., credit scoring), the classifier is not revealed to the agents, which leads agents to attempt to learn the classifier and game it too. In this paper we generalize the strategic classification model to such scenarios. We define the ``price of opacity'' as the difference in prediction error between opaque and transparent strategy-robust classifiers, characterize it, and give a sufficient condition for this price to be strictly positive, in which case transparency is the recommended policy. Our experiments show how Hardt et al.’s robust classifier is affected by keeping agents in the dark.
\end{abstract}

\section{Introduction}
\label{sec:intro}

The increasing role of machine learning in society has led to much recent interest in learning methods that explicitly consider the involvement of human agents. A canonical instance is \emph{strategic classification}---the study of classifiers that are robust to gaming by self-interested agents. This line of research was initiated by \cite{BrucknerS11,hardt2016strategic}, and has quickly amassed a multitude of insightful follow-up works (see below). 
As a running example, consider a firm (e.g., bank) classifying users (e.g., loan applicants) into two classes (e.g., ``approve'' or ``deny''), based on their attributes or features. Since applicants would like their request to be approved (regardless of their true solvency), they might be inclined to spend efforts on modifying their features to align with the bank's decision rule. This can lead to gaming behaviors like holding multiple credit cards, which have nothing to do with the true chances of paying back the loan. The goal of the strategic classification literature is to show when and how the firm can learn a robust classifier, which is guaranteed to perform well even when users strategically modify their features in response to it. But to be realistically applicable, and to ensure safe deployment in settings involving consequential decision-making, care must be taken as to what assumptions underlie the theoretical guarantees. 

Currently known results for strategic classification rely on a key assumption: that users have \emph{full knowledge of the classifier} and respond accordingly, or in other words, the classifier is completely transparent. Under this assumption, the strategic classification literature aims for classifiers with performance ``comparable to ones that rely on secrecy.’’ The literature thus treats the transparency/opaqueness of a classifier as a dichotomy, where only the former is assumed to require a strategy-robust learning process.
However, in many real-life scenarios of high-stake classification, the situation is more nuanced. This point was made already in the original strategic classification paper: ``\emph{Secrecy is not a robust solution to the problem; information about a classifier may leak, and it is often possible for an outsider to learn such information from classification outcomes}'' \cite{hardt2016strategic}. So often we have a hybrid situation – the users know something, not everything, about the classifier, and use their partial knowledge to strategically respond to it. This partial knowledge is a result of a second, smaller-scale learning process, this time of the users themselves, based on information aggregated by media platforms or by their own social network. As anecdotal evidence, consider the secret credit-scoring algorithms applied in the US and elsewhere, coupled with the plethora of online information on how to improve credit scores. In the case of SCHUFA, the German credit-scoring algorithm, there was even a recent crowdsourced effort in which a data donation platform gathered samples from the public
with the goal of ``reverse-engineering'' the secret algorithm (\url{https://openschufa.de/english/}).

The above scenarios demonstrate that rather than preventing gaming, opacity leads users to attempt to learn the classifier and game it too. How relevant is strategic classification to such hybrid situations? What policy recommendations can the theory make for firms whose classifier is being learned?
Our goal in this work is to adapt the strategic classification framework to such settings, in order to examine both theoretically and empirically how keeping users in the dark affects the guarantees of robust classifiers. We introduce the ``optimistic firm’’: It applies a strategy-robust classifier while keeping classification details proprietary. The firm is optimistic in the following sense---it assumes that while gaming on behalf of the users is with respect not to $f$ but to its learned counterpart $\fhat$, the robustness of $f$ still protects it from strategic behavior. We compare how well such a firm fares in comparison to the alternative of simply revealing the classifier. Is the firm justified in its optimistic adoption of strategic classification methods? Our results give a largely negative answer to this question.

{\bf Our Results.}
To address the above questions, we compare the prediction error of a robust classifier $f$ when the strategic users must learn a version of it $\fhat$, and when $f$ is transparent to the users. We term the difference between these errors ``Price of OPacity (POP)'' (Sec.~\ref{sub:POP}). Notice that whenever POP is $>0$, the policy implication is that the transparent policy prevails. 
Our main set of theoretical results (Sec.~\ref{sub:main}) shows the following: We show that even if the estimate $\fhat$ of $f$ is quite good, such that the population mass on which these classifiers disagree is small, these small errors can potentially be enlarged by the strategic behavior of the user population. Indeed, if the small differences in classification incentivize the users to modify their features differently under $f$ and $\fhat$, that means that a much larger mass of users may ultimately be classified differently. We call this enlarged population subset the ``enlargement region’’. 

From the users’ perspective, the enlargement region is undesirable: we show the population mass in this region will be classified negatively whereas under transparency it would have been classified positively (Thm.~\ref{thm:E_partition}). Thus, opaqueness harms those in the region who are truly qualified. From the firm’s perspective, the connection between the enlargement set and POP is not immediately clear – perhaps $\fhat$ is inadvertently fixing the classification errors of $f$, making POP negative? We show a sufficient condition on the mass of the enlargement set for POP to be positive (Thm.~\ref{thm:mainTheorem}). We demonstrate the usefulness of these results 
by analyzing a normally-distributed population classified linearly (Sec.~\ref{sub:linear}). In this setting, we show via a combination of theory and experiments that POP can become very large (Prop.~\ref{prop: POP necessary and sufficient}, Sec.~\ref{sec:experiments}). 
Finally, we formalize the intuition that the problem with keeping users in the dark is that the firm in effect is keeping \emph{itself} in the dark (as to how users will strategically react). We show in Sec.~\ref{sec:all powerful Jury} that if $\fhat$ can be anticipated, a natural generalization of strategic classification holds \citep{Zhang2020IncentiveAwarePL,sundaram2021paclearning}.

We complement our theoretical results with experiments on
synthetic data as well as on a large dataset of loan requests.
The results reinforce our theoretical findings,
showing that \POP\ can be quite large in practice.
We use the loans dataset to further explore the implications of an
opaque policy on users,
showing that it can disproportionately harm users
having few social connections.

{\bf More Related Work.}
%
There is a growing literature on the strategic classification model; to the best of our knowledge, this model has only been studied so far under the assumption of the classifier being known to the agent. 
\citet{DongRSWW18} study an online model where agents appear sequentially and the learner does not know the agents’ utility functions. \citet{CLP19} design efficient learning algorithms that are robust to manipulation in a ball of radius $\delta$ from agents’ real positions (see also \citet{ABBN19}). \citet{MilliMDH19} consider the the social impacts of strategic classification, and the tradeoffs between predictive accuracy and the social burden it imposes (a raised bar for agents who naturally are qualified). \citet{HuIV19} focus on a fairness objective and raise the issue that different populations of agents may have different manipulation costs. \citet{BravermanG20} study classifiers with randomness. \citet{KleinbergR19} study a variant of strategic classification where the agent can change the ground truth by investing effort. \citet{AlonDPTT20,HaghtalabILW20} generalize their setting to multiple agents. \citet{PZM20, BLWZ20} study causal reasoning in strategic classification. \citet{BLWZ20} find that strategic manipulation may help a learner recover the causal variables. 
\citet{RHSP20} provide a framework for learning predictors that are both accurate and that promote good actions.
\citet{scmp}
provide a practical, differentiable learning framework for 
learning in diverse strategic settings.


\ganesh{In concurrent and independent work, \cite{BechavodPZWZ21}  also consider the problem of strategic learning where the classifier is not revealed to the users. In their work, the users belong to distinct population subgroups from which they learn versions of the scoring rule and Jury is aware of these versions (close to our all-powerful Jury setting in Sec.~\ref{sec:all powerful Jury}). We view their work as complementary to ours in that they focus on incentivizing users to put in genuine efforts to self-improve, as advocated by~\cite{KleinbergR19}, rather than on robustness of the classifier to unwanted gaming efforts as in \citep{hardt2016strategic}. One finding of ours that is somewhat related is the difference in how opaque learning affects users who are more vs.~less socially well-connected.}



\section{Preliminaries and Our Learning Setup} 
\label{sec:prelims}

{\bf Classification.}
Let $x \in \X \subseteq \mathbb{R}^d$ denote a $d$-dimensional feature vector describing a user
(e.g., loan applicant), and let $D$ be a distribution over user population $\X$.
Label $y \in \Y = \{\pm 1\} $ is binary (e.g., loan returned or not),
and for simplicity we assume true labels are generated
by  (an unknown) ground-truth function $h(x)=y$.\footnote{The main results of our work in Sec.~\ref{sub:main} hold even when $D$ is a joint distribution over $\X\times \Y$.}
Given access to a sample set $T = \{(x_i,y_i)\}_{i=1}^n$ with
$x_i \iid D$, $y_i=h(x_i)$,
the standard goal in classification is to find a classifier $f:\X \rightarrow \Y$
from a class $\calH$ that predicts well on $D$.

{\bf Strategic Classification (~\cite{hardt2016strategic}).}
%
In this model,
users are assumed to know $f$, 
and to strategically and rationally manipulate their features in order to be classified positively by $f$.
Performance of a classifier $f$ is therefore evaluated on \emph{manipulated data}, in contrast to standard supervised learning, where train and test data
are assumed to be drawn from the same distribution.
In the remainder of this subsection we formally describe the setup of \citet{hardt2016strategic}.
Users gain $1$ from a positive classification and $-1$ from a negative one, and feature modification is costly. Every user $x$ modifies her features to maximize utility. Let:
\begin{equation} \label{eq:best_resp_argmin}
\Delta_f(x) = \argmax_{u \in \X} \{f(u) - c(x,u)\}
\end{equation}
be the utility maximizing feature vector of user $x$, where $c(\cdot,\cdot)$ is a publicly-known, non-negative \emph{cost function} quantifying the cost of feature modification from $x$ to $u$.%
\footnote{Ties are broken lexicographically.} 
It is convenient to think of the mapping $\Delta_f(\cdot)$ as ``moving'' points in $\X$-space.
The goal in strategic classification is to find~$f$ minimizing the \emph{strategic error}
%
$\err(f) := \mathbb{P}_{x\sim D}\{h(x) \neq f(\Delta_{f}(x))\},$
%
which is the probability of the classification error by $f$ when every user $x$ responds to $f$ by modifying her features to~$\Delta_{f}(x)$.



Strategic classification can be formulated as a Stackelberg game between two players:
\emph{Jury}, representing the learner,
and \emph{Contestant}, representing the user population. 
The game proceeds as follows:
Jury plays first and commits to a classifier~$f \in \calH$.
Contestant responds by $\Delta_f(\cdot)$, manipulating feature $x$ to modified feature $\Delta_{f}(x)$. 
%
The payoff to Jury is $1-\err(f)$.
%
To choose $f$, Jury can make use of the \emph{unmodified}
sample set $T=\{(x_i,y_i)\}_{i=1}^n$, as well as knowledge of $\Delta_f$
(i.e., Jury can anticipate how contestant will play for any choice of~$f$). 
The payoff to Contestant is the expected utility of a user $x$ sampled from $D$ and applying modification $\Delta_f(x)$, i.e.,
$\expect{x \sim D}{f(\Delta_f(x)) - c(x,\Delta_f(x))}$. Observe that Contestant's best response to Jury's choice~$f$ is thus $\Delta_f(\cdot)$ as defined in Eq.~\eqref{eq:best_resp_argmin}. We remark that Contestant's utility (and therefore $\Delta_f$) does not directly depend on~$h$. 
It will be convenient to rewrite $\Delta_f(\cdot)$ as follows:
denoting by $C_f(x) = \{ u\mid c(x,u)<2,  f(u)=1\}$
the set of ``feasible''
modifications for $x$ under~$f$ (where the condition that the cost $<2$ comes from $f(x)\in \{-1,1\}$),
\begin{equation*} \label{eq:best_resp}
\Delta_{f}(x) \coloneqq
  \begin{cases}
  \displaystyle \argmin_{u\in C_f(x)}c(x,u) & f(x)=-1 \land C_f(x) \ne \emptyset;\\
  x & \text{otherwise}.
  \end{cases}    
\end{equation*}
In words, $x$ ``moves'' to the lowest-cost feasible point in $C_f(x)$ rather than staying put
if $x$ is (i) classified negatively, and (ii) has a non-empty feasible set.

{\bf Strategic Classification in the Dark.}  
A key assumption made throughout the strategic classification literature
is that \emph{Contestant knows classifier $f$ exactly}, i.e., that $f$ is public knowledge. In this case, we say Jury implements a \emph{transparent policy}. Our goal in this paper is to explore the case where this assumption does not hold, keeping Contestant ``in the dark''. We refer to this as an \emph{opaque policy}.

We consider a \emph{two-sided statistical Stackelberg game},
where both Jury \emph{and} Contestant obtain information through samples.\footnote{We call the conventional strategic learning described above as a \emph{one-sided statistical Stackelberg game}.}
In particular, we assume that since Jury does not publish  $f$,
Contestant is coerced to estimate $f$ from data
available to her
(e.g., by observing the features of friends and
the associated predicted outcomes).
In our theoretical analysis we make this concrete by assuming that
Contestant observes $m$ samples also drawn iid from $D$ but
\emph{classified by $f$},
and denote this set $T_C = \{(x_i,f(x_i))\}_{i=1}^m$. 
In our experiments in Sec.~\ref{sec:experiments}  we consider more elaborate forms of information that Contestant may have access to.
Contestant uses $T_C$ to learn a classifier $\fapx \in \calH_C$
serving as an estimate of $f$,
which she then uses to respond by playing $\Delta_\fapx$.\footnote{Our results also hold in expectation, when the sample sets for the contestant are different across the users.} To allow $\fhat$ to be arbitrarily close to $f$
we will assume throughout that $\calH_C=\calH$ (otherwise the optimal error of $\fhat$
can be strictly positive). 
Intuitively, if $\fhat$ is a good estimate of $f$
(e.g., when $m$ is sufficiently large),
then we might expect things to proceed
as well as if Contestant had known $f$.
However, as we show in Sec. \ref{sec:jitd},
this is not always the case.

We have defined what Contestant knows, 
but to proceed,
we must be precise about what Jury
is assumed to know.
We see two natural alternatives:
(a)~Jury has full knowledge of~$\fhat$. We refer to this setting as 
the \emph{All-Powerful Jury}.
(b)~Jury is unaware of $\fhat$, which we refer to as \emph{Jury in the Dark}.
The latter is our primary setting of interest,
as we believe it to be more realistic and more relevant in terms of policy implications.
It is important to stress that Jury is in the dark
on account of his own decision to be opaque;
any Jury that chooses transparency becomes cognizant of
Contestant's strategy by definition.
We present an analysis of the Jury in the Dark setting  in Secs.~\ref{sec:jitd},\ref{sec:experiments}),
and return to the All-Powerful Jury in
Sec.~\ref{sec:all powerful Jury}.






\section{Jury in the Dark}\label{sec:jitd}


When Jury is in the dark, his knowledge of how Contestant will respond is incomplete,
but he must nonetheless commit to a strategy by choosing~$f$.
We analyze a Jury who assumes
Contestant will 
play~$\Delta_f$.
Indeed, if Jury believes that Contestant's estimated classifier $\fhat$ is a good approximation of~$f$,
then replacing the unobserved $\fhat$ with the known
and similar $f$ is a natural---if slightly optimistic---approach. 
 The optimistic Jury implicitly assumes that small discrepancies between $f$ and $\fhat$ will not harm his accuracy too much. 
This approach also has a practical appeal: 
In practice, many important classifiers are kept opaque, yet firms are becoming more aware to strategic behavior as well as to the possibility of information leakage.
 A firm interested in protecting itself by learning in a manner that is strategic-aware while remaining opaque may choose to apply one of the cutting-edge robust classifiers, and virtually all current methods for strategic classification assume Contestant plays $\Delta_f$.
%
Despite its appeal,
our results indicate that the above intuition can be misleading.
By carefully tracing how errors propagate,
we show that small discrepancies between $f$ and $\fhat$ can `blow up' the error 
in a way that leads to  complete failure of the optimistic approach.

\subsection{Price of Opacity and Related Definitions}
\label{sub:POP}

We are interested in studying the effects of
Jury committing to an opaque policy, as it compares
to a transparent policy.
To this end we extend the 
definition of predictive error: 
\begin{definition}
The strategic error when Jury plays $f$ and 
Contestant responds to  $\fhat$ is given by:
\begin{equation}\label{definition: error}
\err(f,\hat{f}) = \mathbb{P}_{x\sim D}\{h(x) \neq f(\Delta_{\hat{f}}(x))\}.
\end{equation}
\end{definition}
Note that 
$\err(f,f)=\err(f)$; we mostly use $\err(f,f)$ as a reminder
that in principle Jury and Contestant may use different classifiers. 
Our main quantity of interest is:
\begin{definition}[Price of Opacity] \label{def:poo}
When Jury plays $f$ and Contestant responds to  $\fhat$,
the \textbf{Price of Opacity} (\POP) equals to
$\err(f,\fhat) - \err(f,f).$  
\end{definition}
The price of opacity describes the relative loss in accuracy
an optimistic Jury suffers by being opaque, i.e., holding $f$ private.
Note that \POP\ implicitly depends on $h$. Our main results in Sec. \ref{sub:main} show that \POP\ can be strictly positive,
and in some cases, quite large. 

{\bf Enlargement Set. }The two terms in Def.~\ref{def:poo} 
differ in the way Contestant plays---either using $f$ or using $\fhat$---and any difference in the errors can be attributed
to discrepancies caused by this difference.
This provides us a handle to trace the opaque strategic  error and determine price of opacity.
Define the \emph{disagreement region} of $f$ and $\fhat$ as (see Fig.~\ref{fig:E}):
\begin{equation} \label{eq:disagreement_region}
S := \{x \mid f(x) \neq \fhat(x)\}.
\end{equation}
Disagreement stems from errors in Contestant's estimation of $f$.
Such errors become significant if they affect if and how points move
(i.e., $\Delta_f(x)$ vs. $\Delta_\fhat(x)$),
as these differences can lead to classification errors for Contestant.
We refer to the set of all such points as the \emph{enlargement set}: 
\begin{equation} \label{eq:enlargement_set}
E := \{x \mid f(\Delta_f(x)) \neq  f(\Delta_{\fhat}(x))\}
\end{equation}
The enlargement set $E$ includes all points $x$
which Jury classifies differently under the different 
strategies of Contestant (see Fig.~\ref{fig:E}).
In particular, $E$ tells us which points \emph{prior to moving} will be disagreed on
\emph{after moving}. Because all moved points tend to concentrate around the
classification boundary (since they maximize utility),
we think of $E$ as `enlarging' $S$ by including  points that strategically moved to  $S$. 
This entails a subtle but important point:
even if $\fhat$ is a good approximation of $f$
and the mass on $S$ is small,
the mass of $E$, which includes
points that \emph{map to $S$} via $\Delta_f$,
can be large. In Sec.~\ref{sub:main} we characterize the points in $E$. 

{\bf Errors $\epsilon_1,\epsilon_2$ in Strategic Learning.}
Let $\epsilon_1,\epsilon_2$ be such that the learning error 
for Jury is:
$\err(f,f)  = \err(\fstar,\fstar) + \epsilon_1,$
where $\fstar = \argmin_{f \in \calH} \err(f,f)$
is the optimal strategy-aware classifier, and the learning error for the Contestant is:  
$\pr_{x \in D} \{ x\in S\} = \epsilon_2.$
In the next sections we use $\epsilon_1$ and $\epsilon_2$
to reason about \POP.
These errors will tend to zero (with rates polynomial in the number of samples):
for $\epsilon_1$ when $f$ is strategically-PAC learnable (see \cite{Zhang2020IncentiveAwarePL,sundaram2021paclearning}),
and for $\epsilon_2$ when $\fhat$ is learnable
in the standard, non-strategic PAC sense (recall that $\calH_{C}=\calH$).

{\bf Relations Between Errors.}
We note that
the relations between $\err(f,\fhat)$, $\err(f,f)$,
and $\err(\fstar,\fstar)$ are not immediate.
First, we observe that 
there can be instances in which 
$\err(f,\fhat) < \err(\fstar,\fstar)$,
i.e,  the utility gained by a \naive\ and opaque Jury
could be \emph{higher} than that of the optimal transparent Jury.
Second, we note that $\POP$ can be \emph{negative}
as there exist instances where $\err(f,\fhat) < \err(f,f)$. We demonstrate this via experiments in App.  \ref{apndx:experiments}.

Intuitively, we expect that if $n$ is large
($\epsilon_1$ is small)
and $m$ is small ($\epsilon_2$ can be large),
then estimation errors of $\fhat$ would
harm Jury's ability to predict Contestant's moves,
so much that it would lead to a strictly positive \POP.
This is indeed the intuition we build on for our next
result which determines when $\POP>0$. 
Additionally, even if $\epsilon_2$ is small, 
the probability on $E$ can be large.  

\subsection{Main Results}
\label{sub:main}

 We partition the points in $E$ depending on how $h$ behaves on it.
 This partition is used to first determine \POP\  and then give a sufficient condition for positive \POP\ (see Thm.  \ref{thm:mainTheorem}). Next, in Thm.~\ref{thm:E_partition}\,   we give an exact characterization of  $E$ in terms of  $S$, $f$, $\widehat{f}$ (independent of $h$).   Together, these results are
 useful in analysis, when one needs to assess whether the price of opacity exists  under a reasonable distribution model $D$, as we demonstrate in Sec.~\ref{sub:linear}. 

{\bf Partition of $E$.} The points in $E$ can be partitioned as:
\begin{align*}
    E^{+} &=  \{x \mid h(x)=f(\Delta_{f}(x)) \ \land \ h(x) \neq f(\Delta_{\fhat}(x)) \}; \\
    E^{-} &=  \{x \mid h(x) \neq f(\Delta_{f}(x)) \ \land \ h(x) =  f(\Delta_{\fhat}(x)) \}.
\end{align*}
To make concrete the relation between \POP, $E$ and~$h$ define:
\begin{equation*}
\POP^{+} := \mathbb{P}_{x\sim D}\{x\in E^{+}\}, \,\,\,\,
\POP^{-} := \mathbb{P}_{x\sim D}\{x \in E^{-}\}.
\end{equation*}
Notice that, in Fig.~\ref{fig:E}, the points in region $\POP^{+}$ increase the price of opacity, and the points in region $\POP^{-}$ decrease the price of opacity. Further, note that since $E^{+}$ and $E^{-}$ form a disjoint partition of $E$, we have  
\begin{equation}
        \mathbb{P}_{x\sim D}\{x\in E\} = \POP^{+} +  \POP^{-}. \label{equation: first characterization of E}
\end{equation}
The next lemma follows from the definition of \POP. 
\begin{lemma}\label{lem: characterization of POP}
$ \POP = \POP^{+}  -   \POP^{-} \label{equation: relation of pop pop+ pop-}.$
\end{lemma}
\ganesh{
Note that since \POP\ is a difference between two
probabilities, its range is $\POP \in [-1,1]$.
} 

\begin{figure*}
    \centering
    \begin{subfigure}{ }
    \includegraphics[width=0.23\textwidth]{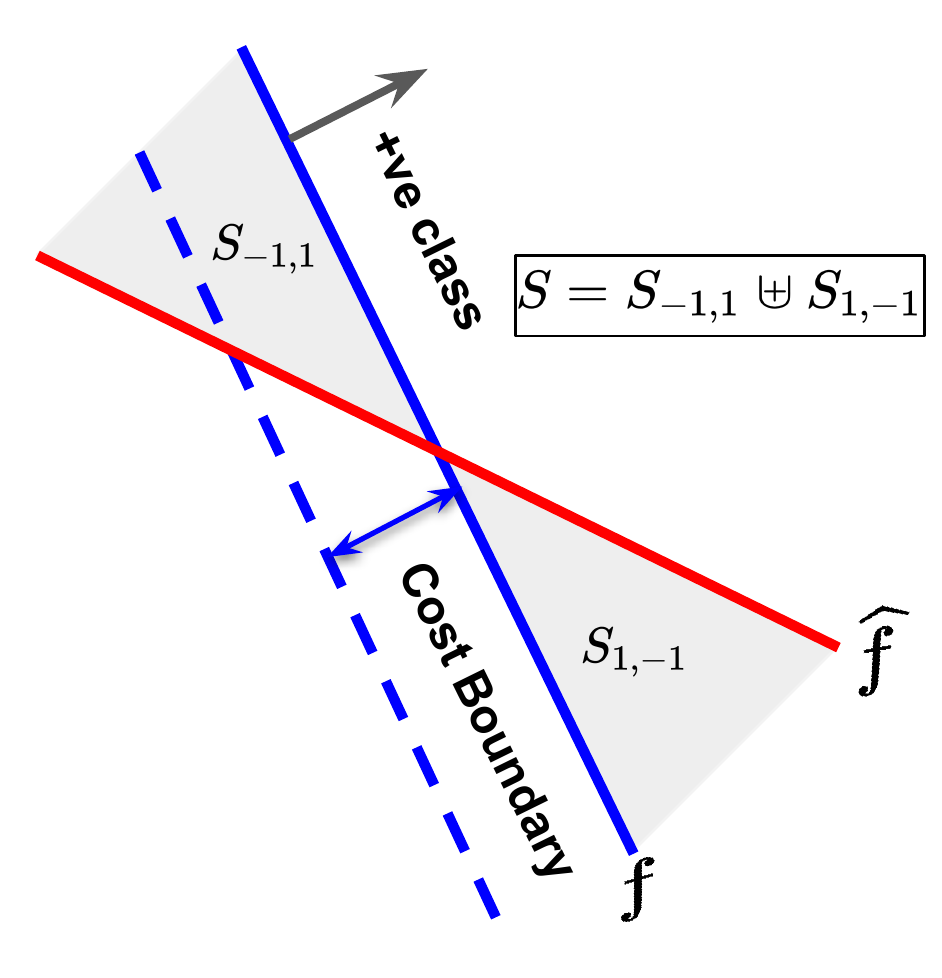}
    \end{subfigure}
    \,
    \begin{subfigure}{ }
    \includegraphics[width=0.23\textwidth]{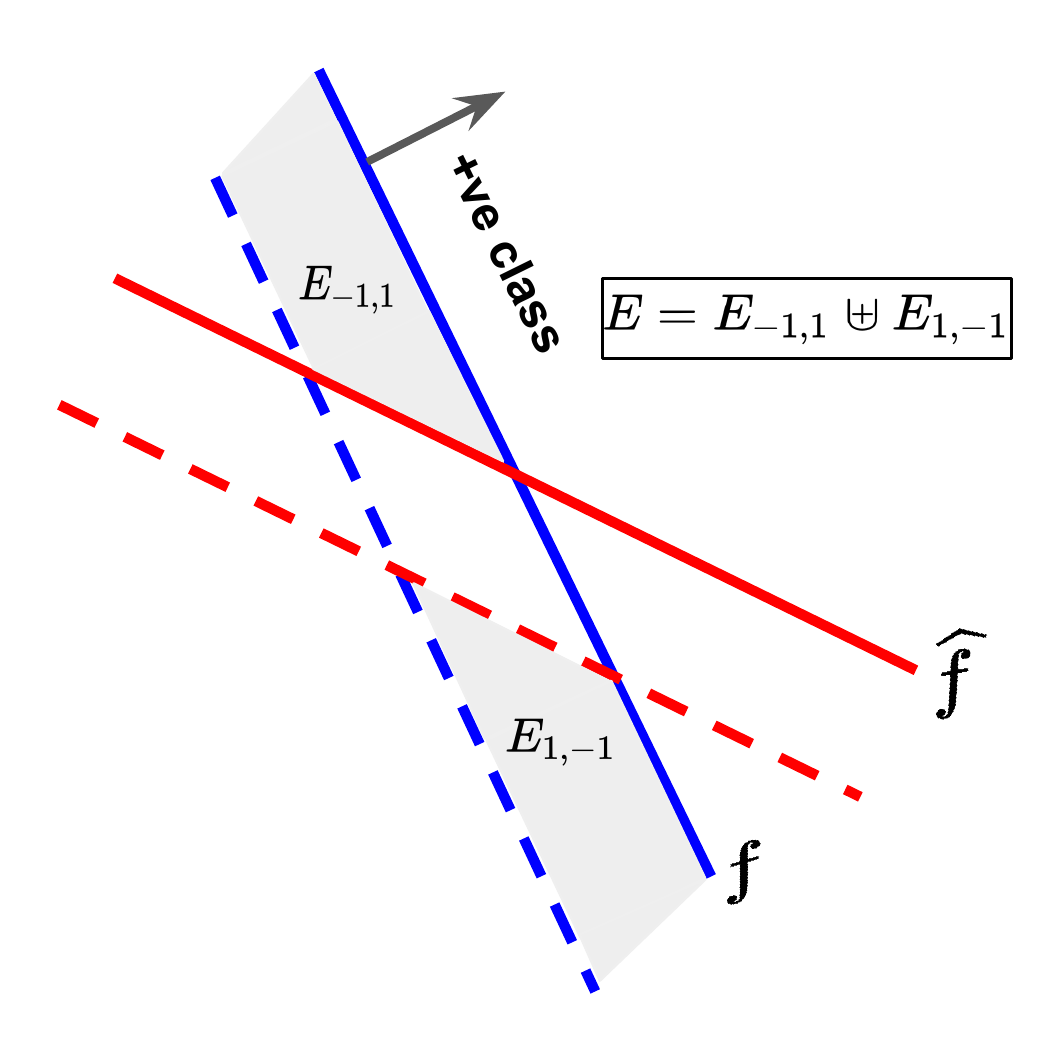}
    \end{subfigure}
    \,
    \begin{subfigure}{ }
    \includegraphics[width=0.23\textwidth]{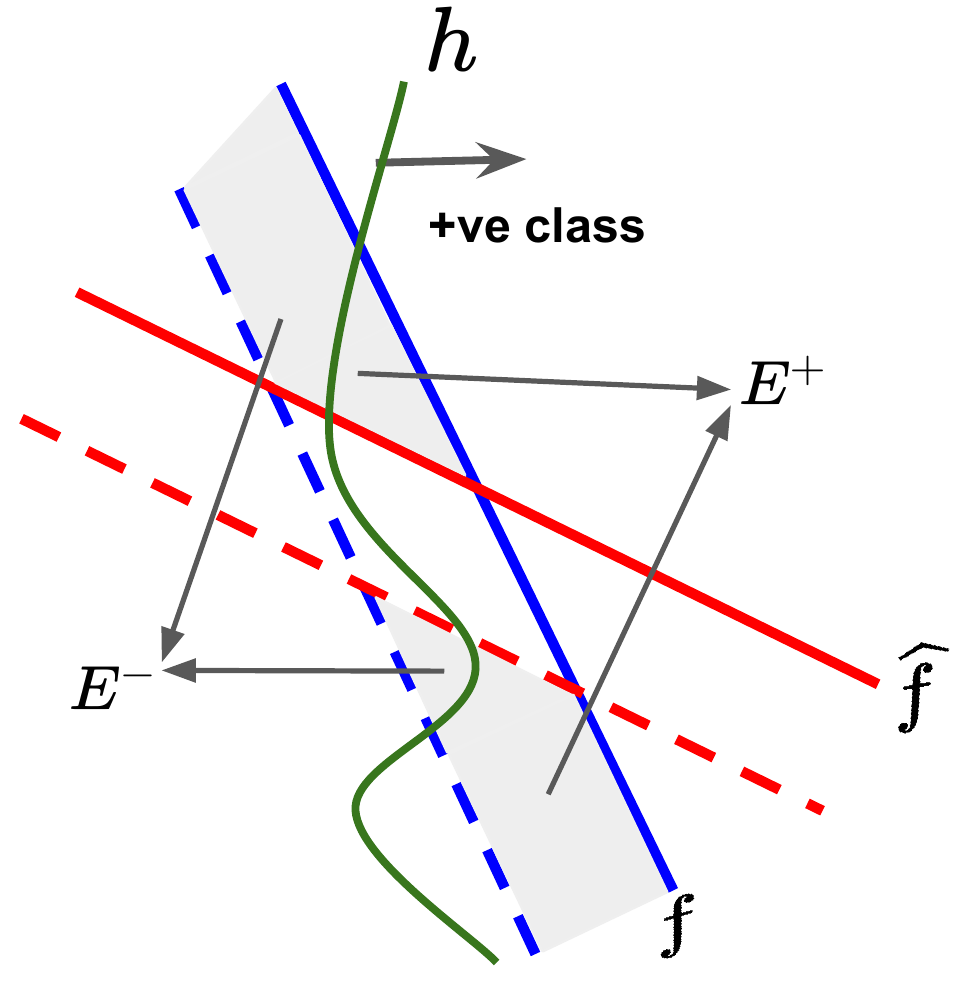}
    \end{subfigure}
        \caption{\textbf{(Left)} The disagreement set $S$ between $f$ and $\fhat$, and the cost boundary of $f$.
        \textbf{(Center)} The enlargement set $E$ derived from $S$.
        \textbf{(Right)} The partition $E^+$ and $E^{-}$ of $E$. Note the following: a) $E$ is a subset of the cost boundary of $f$, and b) $E^{+}$ and $E^{-}$ depend on $h$.}
    \label{fig:E}
\end{figure*}

We can now state our main result, proved using Lem. \ref{lem: characterization of POP}.
The theorem provides a sufficient condition for $\POP>0$:
if the probability mass on $E$ exceeds the base
error rates of $f$ and $\fstar$,
then \POP\ is strictly positive.
This provides a handle for reasoning about the price of opacity
through the analysis of the enlargement set $E$.
We remark that even though \POP\ depends on the behaviour of $h$ on $E$, it is interesting that the sufficiency condition in Thm. \ref{thm:mainTheorem} depends only on $f$ and $\fhat$ which determine $E$. We also show in Sec. \ref{sub:linear} (see Cor. \ref{corollary: POP >0}) that there are instances where the sufficiency condition in the theorem below is indeed  necessary.
\begin{restatable}{theorem}{mainThm}
\label{thm:mainTheorem}
The price of opacity is strictly positive, i.e., $\POP > 0$, if the following condition holds:
\begin{equation}
\mathbb{P}_{x\sim D}\{x\in E\} > 2\err(f^{\star},f^{\star}) + 2 \epsilon_1.
\end{equation}
\end{restatable}
The mass of the enlargement set is potentially large if the distribution $D$ has a large mass on the ``cost boundary'' of $f$ (see Fig.~\ref{fig:E}).
Intuitively, enlargement can cause a ``blow-up'' in  \POP\
because both learning objectives---of $f$ by Jury,
and of $\fhat$ by Contestant---are oblivious to the significance of this mass under strategic behavior.

{\bf Characterizing the Enlargement Set $E$.}
We  identify the points in the enlargement set~$E$ when only  $f$ and $\fhat$ are known without invoking an explicit dependence on the (often unknown) ground truth function $h$. 
First partition $S$~as:
\begin{align*}
S_{-1, 1} &= \{x \mid f(x)=-1 \land \fhat(x) = 1\};\\
S_{1, -1} &= \{x \mid  f(x)=1 \land \fhat(x) = -1\}.    
\end{align*}
Define the following sets derived from $S_{-1, 1}$ and~$S_{1, -1}$:
\begin{align*}
E_{-1,1} &= \{x \mid \Delta_{\fhat}(x)\in S_{-1,1},~ f(\Delta_{f}(x)) =1\};\\
E_{1,-1} &= \{x \mid \Delta_{f}(x)\in S_{1,-1}\setminus \{x\},~ \Delta_{\fhat}(x)=x \}.
\end{align*}
In Thm.  \ref{thm:E_partition} (proof in App.  \ref{apndx:missing proofs}) we show that $E_{-1,1}$ and $E_{1,-1}$ partitions $E$. We remark that this partition is different from $E^{+}$ and $E^{-}$ as defined before (see Fig.~\ref{fig:E}).
\begin{restatable}{theorem}{ECharacterization}
\label{thm:E_partition}
The enlargement set can be partitioned as
$E = E_{-1,1} \uplus E_{1,-1}$.
\end{restatable}

One consequence of Thm.~\ref{thm:E_partition} is that while a transparent Jury would have classified  users in $E$ positively,
an opaque Jury classifies them negatively. This disproportionately affects qualified users.  
\begin{corollary}
\label{corollary: the sign of points in the disagreement set}
For all $x\in E$, it holds that
$f(\Delta_f(x)) = 1 \text{ and } f(\Delta_{\fhat}(x)) = -1.$
\end{corollary}
\subsection{Price of Opacity for Linear Classifiers}
\label{sub:linear}
In this section, we use results from Sec.~\ref{sub:main}\ to exemplify how the strategic behaviour by Contestant according to her learnt classifier $\fhat$ (instead of $f$) can harm Jury and lead to positive \POP. Throughout this section, we make the simplifying assumption that $d=1$ and that both Jury and Contestant play linear (threshold) classifiers, so as to make the computations simpler to follow and highlight Contestant's strategic behaviour. Our main findings on \POP\ for linear classifiers, Gaussian distributions over $\X$, and a large class of cost functions (as in Def.  \ref{def:admissible_cost}) are the following.
First, using  Thm.~\ref{thm:E_partition} and Cor.~\ref{corollary: the sign of points in the disagreement set} we determine $E$ (applying knowledge of $f$ and $\fhat$), and use Thm.~\ref{thm:mainTheorem} to claim that $\POP\ >0$ is inevitable (for any $h$) if the probability mass on $E$ is large (Prop.~\ref{prop: POP linear classifiers sufficiency condition}).  In  Prop.~\ref{prop: POP necessary and sufficient} we show the realizablility of $h$ enables us to use Lem.~\ref{lem: characterization of POP} to determine the expression for \POP\ in this setting.\footnote{Note that such knowledge regarding $h$ is necessary 
as \POP\ depends on the behaviour of $h$ on $E$.} Next, we show that the sufficiency condition in Thm.~\ref{thm:mainTheorem} is tight by giving a class of $h$ (realizable) for which this condition is also necessary (Cor.~\ref{corollary: POP >0}).  This establishes that without  further assumptions on $h$ it is impossible to derive a better condition for $\POP>0$. Finally in Cor.~\ref{cor: POP  linear classifiers}, with explicit knowledge of $h$ we show \POP\ can be made arbitrarily
large (i.e., close to 1)  with
 an appropriate choice of distribution parameters.

Our results here apply to cost functions that are \emph{instance-invariant}
\citep{sundaram2021paclearning},\footnote{Our focus on instance-invariant costs
is due to a result from \citet{sundaram2021paclearning}
showing that \emph{non}-invaraint costs (referred to as\emph{instance-wise} costs)
are intractable even in simple cases.
For example, the Strategic-VC of linear classifiers
is linear for instance-invariant costs
but unbounded for instance-wise costs.}
with an additional feasibility requirement
which ensures that each user has a feasible modification.
We refer to these as \emph{admissible costs}.
%
\begin{definition}[Admissible Cost]
\label{def:admissible_cost}
A cost function $c: \X \times \X \rightarrow \mathds{R}_{+}$ is \emph{admissible} if:
\newline 
\-\ (i) $c(x, x+ \delta) = c(y, y+\delta)$ for all $x,y,\delta \in \X$, and
\newline 
\-\ (ii) for all  $x\in \X$ there exists $\delta >0 $ s.t. $c(x, x+\delta) \leq 2$
\end{definition}
In the above, condition (i) is equivalent to instance-invariance,
and condition (ii) ensures feasibility.

\vineet{
Let $c$ be an admissible cost function $c$,
and define $t := \sup_{\delta> 0}\{c(x,x+\delta) \leq 2\}$.
Let $D= \mathcal{N}(\alpha,\sigma)$ be a Gaussian distribution over $\X$ with $\Phi_{\sigma}(\cdot)$ as its CDF.
}
The thresholds corresponding to  Jury and  Contestant are $t_f$ and $t_{\fhat}$ respectively,
\ganesh{
i.e., Jury’s classifier is $f(x)=\sign(x \geq  t_f)$
and Contestant's classifier is $\widehat{f}(x)=\sign(x \geq  t_{\widehat{f}})$.
}
We assume without loss of generality that $\epsilon_2 > 0$ and $t_f \neq t_{\fhat}$. In Prop.~\ref{prop: POP necessary and sufficient} and Cor.~\ref{corollary: POP >0}, when we say $h$ is realizable, we assume that $h(x) =1$ for $x\geq \alpha$ and $-1$ otherwise.
We now present the first result which gives the sufficient condition for
$\POP>0$. In Prop.~\ref{prop: POP linear classifiers sufficiency condition} let $\ell:= 2 \cdot \err(\fstar,\fstar) + 2\epsilon_1 $.
\begin{restatable}{proposition}{proptwo}
\label{prop: POP linear classifiers sufficiency condition} 
The following are sufficient for $\POP>0$:\\
(a) \,\, $\Phi_{\sigma}(t_f) -  \Phi_{\sigma}(t_f-t)  > \ell$
 when $t_f > t_{\fhat}$\\
(b) \,\, $\Phi_{\sigma}(t_{\fhat}-t) -  \Phi_{\sigma}(t_f-t)  >  \ell$
when $t_f < t_{\fhat}$
\end{restatable}
Note that $t_f > t_{\fhat} \Rightarrow E = [t_f - t, t_f]$
and $t_f < t_{\fhat} \Rightarrow  E = [t_f-t, t_{\fhat}-t]$,
and that if $h$ is realizable, then $\err(\fstar,\fstar) = 0$.  Using this fact and assuming a reasonable relation between $\epsilon_2$ and $\epsilon_1$ we show there exists $\sigma_0$ such that $\POP>0$ for all $\sigma < \sigma_0$,
and that \POP\ can be made arbitrarily \ganesh{close to 1} by reducing $\sigma$.
\begin{restatable}{proposition}{propnecessaryandsufficient}
\label{prop: POP necessary and sufficient}
Let $h$ be realizable and $\epsilon_2 > 2\epsilon_1$. Then, there exists $\sigma_0 >0$ such that for all $\sigma < \sigma_0$ it holds that:
\begin{equation}
\label{eqn: POP characterization equation}
\POP = \begin{cases}
 \Phi_{\sigma}(t_{\fhat} - t) - \Phi_{\sigma}(|t -t_f|) & \text{if } t_f < t_{\fhat}\\
 \Phi_{\sigma}(t_f) - \Phi_{\sigma}(|t-t_f|)
& \text{if } t_f > t_{\fhat}
\end{cases}
\end{equation}
%
\end{restatable}
We use this to demonstrate instances where the sufficiency condition in Thm.~\ref{thm:mainTheorem} is also necessary.
\begin{restatable}{corollary}{corollaryforPOP>0}
\label{corollary: POP >0}
Suppose $h$ is realizable, $t_f < t$, and $\epsilon_2 > 2\epsilon_1$. Then there exists $\sigma_0 >0$ such that for all $\sigma < \sigma_0$, $\POP > 0$ if and only if $\pr_{x\sim D}\{E\} > 2\epsilon_1$.
\end{restatable}

Recall, $\epsilon_2$ is the probability mass of the disagreement set, and intuitively, $\epsilon_2$ is large if $m$ is small.  Moreover,    the strategic VC dimension of linear classifiers in $\mathbb{R}^{d}$ is at most $d+1$ for admissible cost functions (\citet{sundaram2021paclearning}) and hence  threshold classifiers are  strategic PAC learnable. 
In particular, with probability at least $1-\delta$ the following holds, which we use for Cor.~\ref{cor: POP  linear classifiers}:
$\epsilon_1 \leq \sqrt{ \log \frac{4}{\delta} / n }.$
%
%
%
%
%
\begin{restatable}{corollary}{CorollaryTwo}
\label{cor: POP  linear classifiers}
Suppose $ \epsilon_2 >  2\sqrt{ \log \frac{4}{\delta} / n }$, where $\delta \in (0,1)$. Then  with probability at least $1-\delta$, $\POP$ is as in Eq.~\eqref{eqn: POP characterization equation}.
\end{restatable}

{\bf Discussion.}
As a take-away from our analysis in this section, say the number of Contestant's samples $m$ is fixed, and the number of Jury's samples $n$ grows. Then the following trade-off occurs: On one hand, as Jury acquires more training samples, his strategic error $\err(f,f)$ decreases (since $\epsilon_1 \leq \sqrt{ \log \frac{4}{\delta} / n }$). On the other hand, as $n$ grows, the sufficient condition on $\epsilon_2$ becomes weaker (so positive \POP\ occurs more), and the expression for POP grows. So increasing the relative gap between $n$ and $m$ is likely to increase \POP, but discarding some of Jury's training points to decrease \POP\ is likely to increase the strategic error. 

\section{Experiments}
\label{sec:experiments}
In this section we complement our theoretical results with an experimental evaluation of \POP.
We begin with a synthetic experiment that validates our results from Sec.~\ref{sub:main}  and ~\ref{sub:linear}.
We then proceed to analyzing a real dataset of loan applications in the peer-to-peer lending platform Prosper (\url{http://www.prosper.com}).
This dataset is unique in that it includes:
(i) logged records of loan applications,
(ii) corresponding system-evaluated risk scores, and
(iii) a social network connecting users of the platform, with links portraying social (rather than financial) relations \citep[see][]{krumme2009lending}. 
We use loan details as user features and risk scores as labels to study \POP\ behavior,
and social network connections to perform an in-depth analysis of the effects of an opaque policy.
Code is publicly available at
\url{https://github.com/staretgicclfdark/strategic_rep}.


{\bf Experimental Setting.}
In all experiments we use the algorithm of \citet{hardt2016strategic} as the strategic learner of Jury
and consider settings for which this algorithm provides learning guarantees
(i.e., a separable cost function and deterministic labeling).
Our main experiments model Contestant as inferring $\fhat$
using a simple ERM approach (in practice, SVM).
We explore other risk-averse strategies for contestant
in Appendix \ref{secappendix: loan experiment addtional result}.
Although our theoretical analysis considered a single $\fhat$ for Contestant,
as we note, our results hold in expectation for the more realistic case in which
each $x$ is associated with a user, and each such user has access to her own sample set $T_C(x)$
of size $m$, on which she trains an individualized $\fhat_x$.
This is the setting we pursue in our experiments.

Because we will consider very small values for $m$, there is a non-negligible chance
that $T_C(x)$ will include only one class of labels.
A plausible modeling assumption in this case is to assert that points with one-sided information do not move.
However, to prevent the attribution of \POP\ to this behavior, we instead choose to
ensure that each of Contestant's sample sets includes at least one example from each class,
for which we use rejection sampling.
App.~\ref{app:experiments} includes additional information on our setup.


\subsection{Split Gaussian}
Following the setup in Sec.~\ref{sub:linear},
we generate data using $D = \mathcal{N}(0,1)$ and $h(x)=\sign(x \ge 0)$.
We set $c(x,x')=\max\{0,x'-x\}$.
We first sample $n=5000$ labeled points and split them 80-20
into a train set $T$ and a held-out test set $S$,
and use $T$ to train Jury's classifier $f$.
For every $x\in S$, we repeatedly sample an additional $m$ points labeled by $f$
to construct training sets $T_C(x)$,
and use these to train  $\fhat_x$ for every $x \in S$ to be used by Contestant.
We repeat this for  $m\in[4,4096]$.
\begin{wrapfigure}{r}{0.42 \textwidth}
	\centering
	\includegraphics[width=0.45\textwidth]{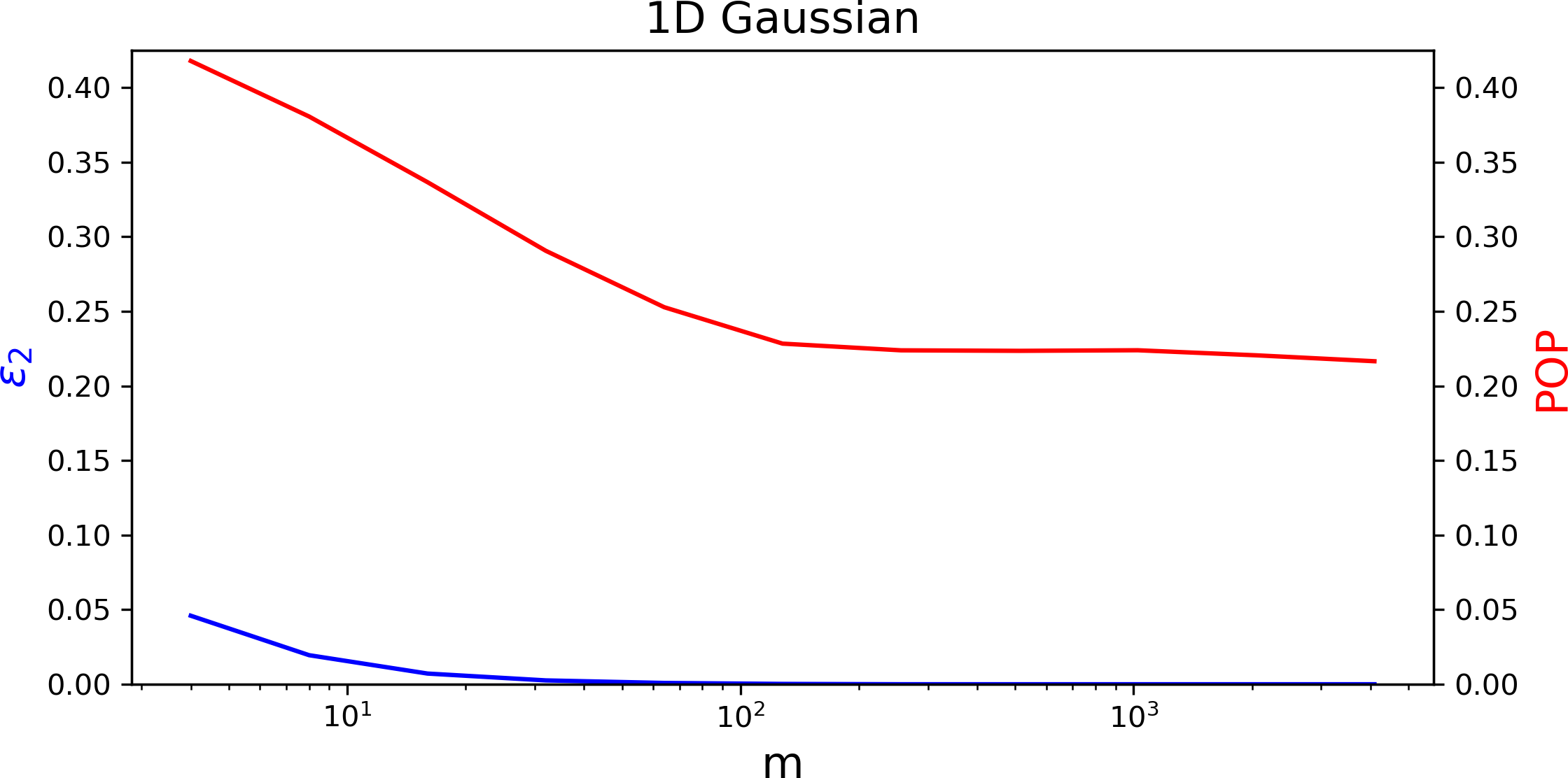}
	\caption{\POP\ and estimation error for $\fhat$ ($\epsilon_2$) on a split
	1D Gaussian.
	\POP\ is large even when $\fhat$ closely estimates $f$.
 	\vspace{-0.5cm}
	}
	\label{fig:1d}
\end{wrapfigure}
Fig.~\ref{fig:1d} shows how \POP\ varies with $m$ (log scale),
along with  corresponding estimation error of $\fhat$ w.r.t. $f$ (i.e., $\epsilon_2$).
 For small $m$ the price of opacity is extremely large ($\sim 0.5$).
As $m$ increases, \POP\ decreases, but quickly plateaus at $\sim 0.3$ for $m \approx 100$.
Note that the estimation error of $f$ is small even for small $m$,
and becomes (presumably) negligible at $m \approx 30$, before \POP\ plateaus.


\subsection{Loan Data} \label{sec:exp_loans}
We now turn to studying the Prosper loans dataset.
We begin by analyzing how \POP\ varies as $m$ increases,
sampling uniformly from the data.
The data includes $n=20,222$ examples,
which we partition $70-15-15$ into three sets:
a training set $T$ for Jury,
a held-out test-set $S$,
and a pool of samples from which we sample points for each $T_C(x), x\in S$.
We set labels according to (binarized) system-provided risk scores,
and focus on six features that are meaningful and that are amenable to modification: available credit, amount to loan, \% trades not delinquent, bank card utilization, total number of credit inquiries, credit history length.
A simple non-strategic linear baseline using these features achieves 84\% accuracy
on non-strategic data.
We compare performance of the algorithm of \citet{hardt2016strategic} (\ganesh{\texttt{HMPW}})
to a non-strategic linear SVM baseline. Both models are evaluated under three Contestant types:
\emph{Non-Strategic}, where points do note move;
\emph{Fully-Informed}, where Contestant plays $\Delta_f$; and
\emph{In the Dark}, where Contestant plays $\Delta_\fhat$.

\begin{figure*}[t]
	\centering
	\includegraphics[width=0.44\textwidth]{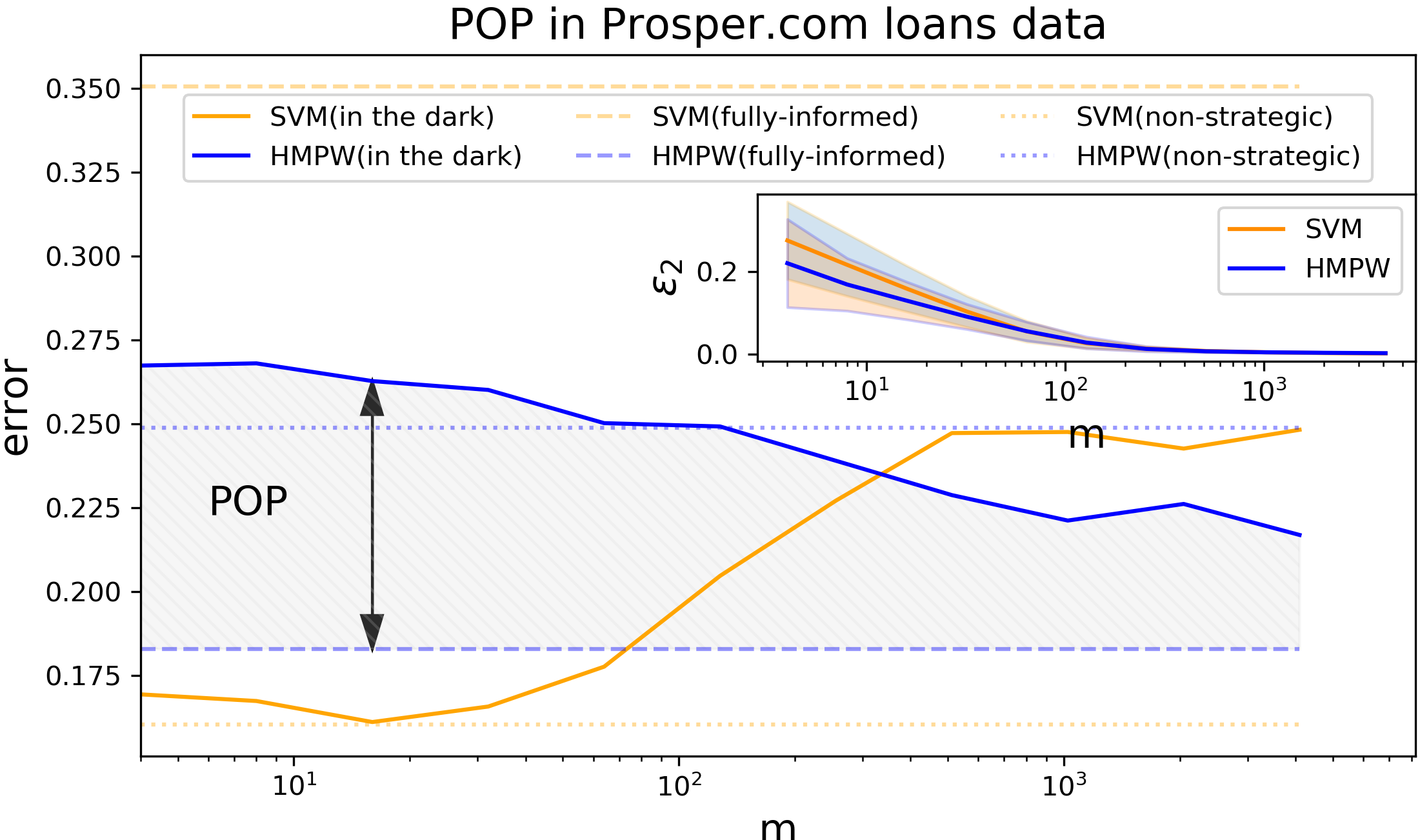} \,
	\includegraphics[width=0.26\textwidth]{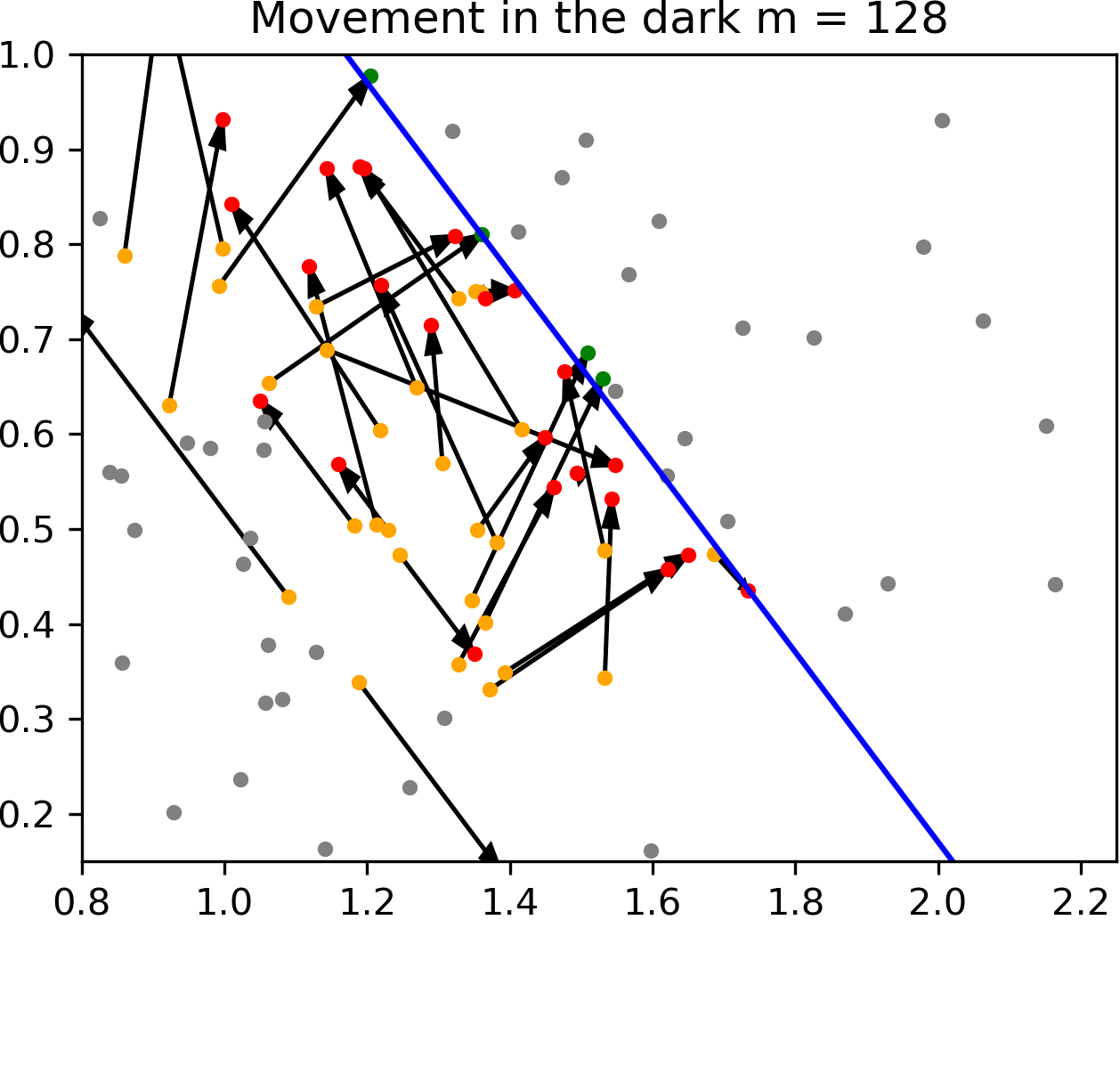} \,
	\includegraphics[width=0.26\textwidth]{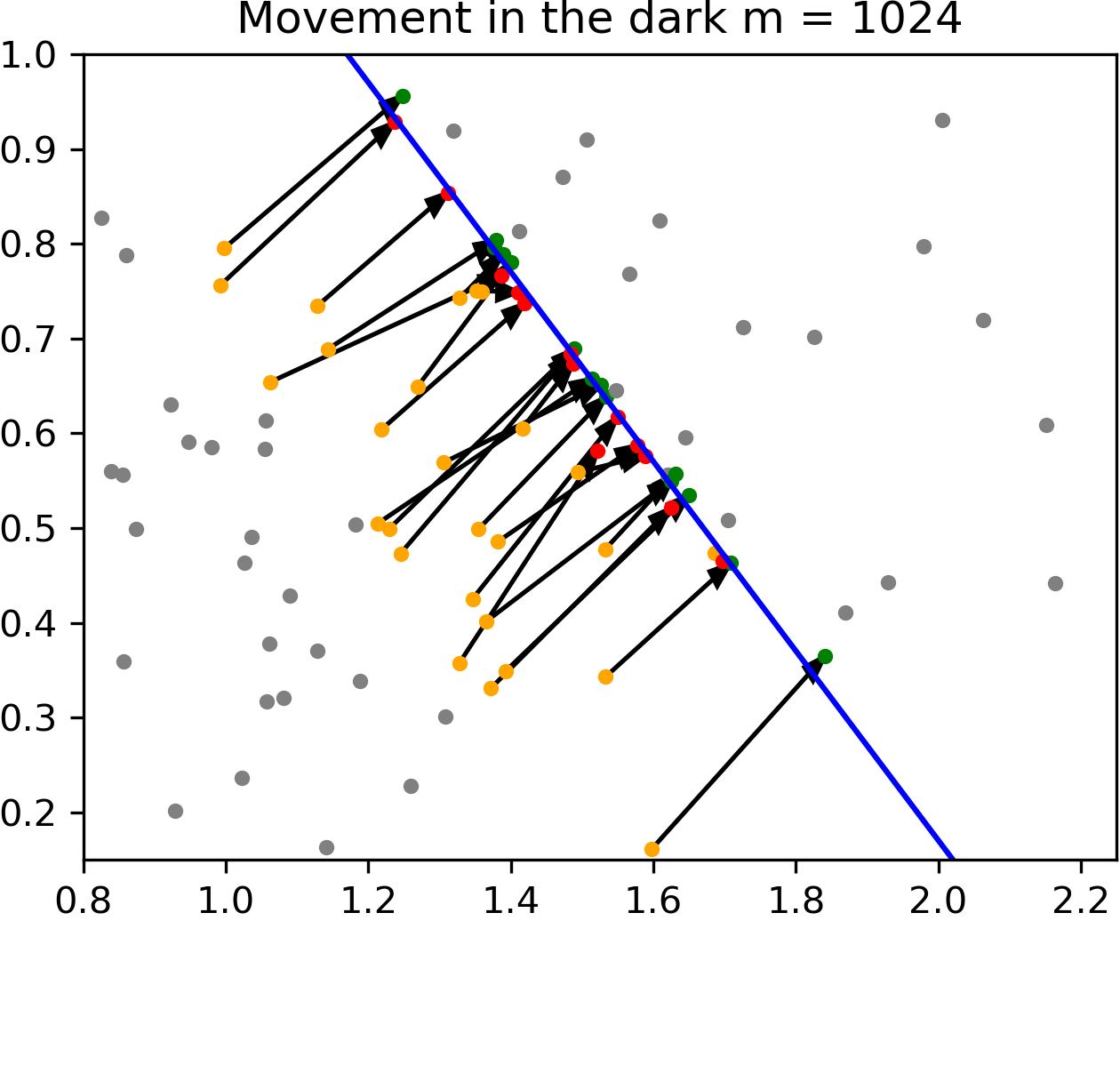}
	\caption{\textbf{(Left)} Prediction errors on loans data for three user types:
	\emph{non-strategic},
	\emph{fully-informed}, and \emph{in the dark}.
	\POP\ decreases as $m$ increases but remains considerable
	even for large $m$.
	\textbf{(Center+Right)} Illustrations of feature modifications
	for $m=128$ and $m=1024$, projected to 2D.
	Blue line is the (projected) classifier threshold,
	with points to its left classified as positive.
	Features that were not modified are in grey,
	and features that were modified are in orange,
	with arrows mapping to the modified features,
	with green and red indicating positive and negative outcomes, respectively.
	}
	\label{fig:prosper1}
\end{figure*}

{\bf Price of Opacity.}
Fig.~\ref{fig:prosper1} (left) shows how predictive performance varies for increasing $m$.
\POP\ is the difference between lines corresponding to the
In the Dark and Fully-Informed types (shaded area for \ganesh{\texttt{HMPW}}).
In the fully-informed case ($\Delta_f$), the performance of \ganesh{\texttt{HMPW}}
closely matches the benchmark of SVM on non-strategic data,
showing that in this case, by correctly anticipating user behavior,
\ganesh{\texttt{HMPW}} is able to a account for gaming in learning.
However, when Contestant is in the dark ($\Delta_\fhat$),
for small $m$ the performance of \ganesh{\texttt{HMPW}} is as bad as when Contestant does not move at all,
and the price of opacity is large ($\sim 9\%$).
As $m$ increases,
Contestant better estimates $\fhat$,
and the performance of \ganesh{\texttt{HMPW}} increases.
However, even for very large $m$ (e.g., $m=4096$),
\POP\ remains substantial ($\sim 4\%$).

Fig.~\ref{fig:prosper1} (center; right) visualizes how points move under \ganesh{\texttt{HMPW}}
when Contestant is \emph{In the Dark},
for medium-sized sample sets ($m=\ganesh{128}$) and large sample sets ($m=\ganesh{1024}$).
For visualization we project points down to $\R^2$ in a way that aligns with the predictive model $f$,
ensuring that points are classified as positive iff they appearing above the line.
The figures demonstrate which points move and where.
For $m=\ganesh{128}$, some points move in the right direction, others move in seemingly arbitrary directions.
For $m=\ganesh{1024}$, most points move in the right direction, but in many cases fall short
of the classification boundary.

{\bf Opacity and Social Inequity.}
Our analysis thus far has been focused on the effects of an opaque policy to the payoff of Jury.
But an opaque policy primarily keeps \emph{users} in the dark,
and here we analyze the effects of opacity on the payoff to users.
When Jury anticipates strategic behavior, he must `raise the bar' for positive classification
to prevent negative-labeled points from gaming.
But due to this, positive-labeled points must also make effort to be classified as such.
Ideally, Jury can set up $f$ in a way that positive-labeled users are classified appropriately
if they invest effort correctly.
This, however, may require exact knowledge of $f$,
and when positive-labeled users are in the dark, the lack of access to $f$
may cause them to be classified negatively despite their effort in modification (see Cor. \ref{corollary: the sign of points in the disagreement set}).

Fig.~\ref{fig:prosper2} (right) shows for increasing $m$
the percentage of positive-labeled users
that are classified correctly under full information (i.e., $f(\Delta_f(x))=1$),
but are classified incorrectly when in the dark  (i.e., $f(\Delta_\fhat(x))=-1$);
this is precisely the positive enlargement set $E^+$ in
Cor. \ref{corollary: the sign of points in the disagreement set}.
As can be seen, for low values of $m$, roughly $50\%$ of these users are classified incorrectly.
When considering a heterogeneous population of users varying in the amount of information
they have access to (i.e., varying in $m$),
our results indicate a clear form of inequity towards individuals
who are not well-connected.
This is a direct consequence of the opaque policy employed by Jury.

{\bf Stories from the Network.}
We now dive into the Prosper.com social network
and portray stories describing individual users and how their
position within the network affects how they are ultimately classified---and hence whether their loan is approved or denied---under an opaque policy.
Fig.~\ref{fig:prosper2} (left) presents five such cases.
The figure shows for each user her initial features $x$
and her modified features $\Delta_\fhat(x)$ (if they differ).
We assume users observe the features and predictions of their neighbors in the network up to two hops away, which they use to learn $\fhat$.
As before, points are projected down to $\R^2$,
so that points above the line are approved and below it are denied (to reduce cluttering, projections
slightly vary across users).
The plot also shows network connections between users (superimposed).
\begin{figure*}[t]
	\centering
	\includegraphics[width=0.62\textwidth]{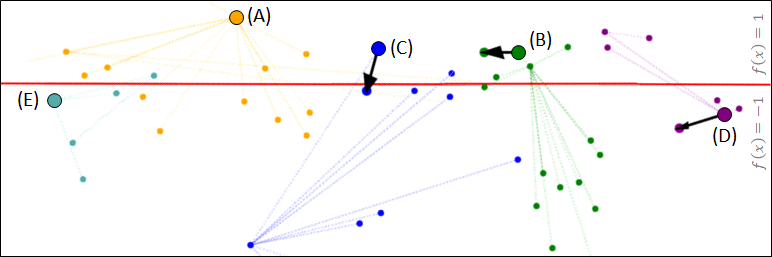} \qquad
	\includegraphics[width=0.18\textwidth]{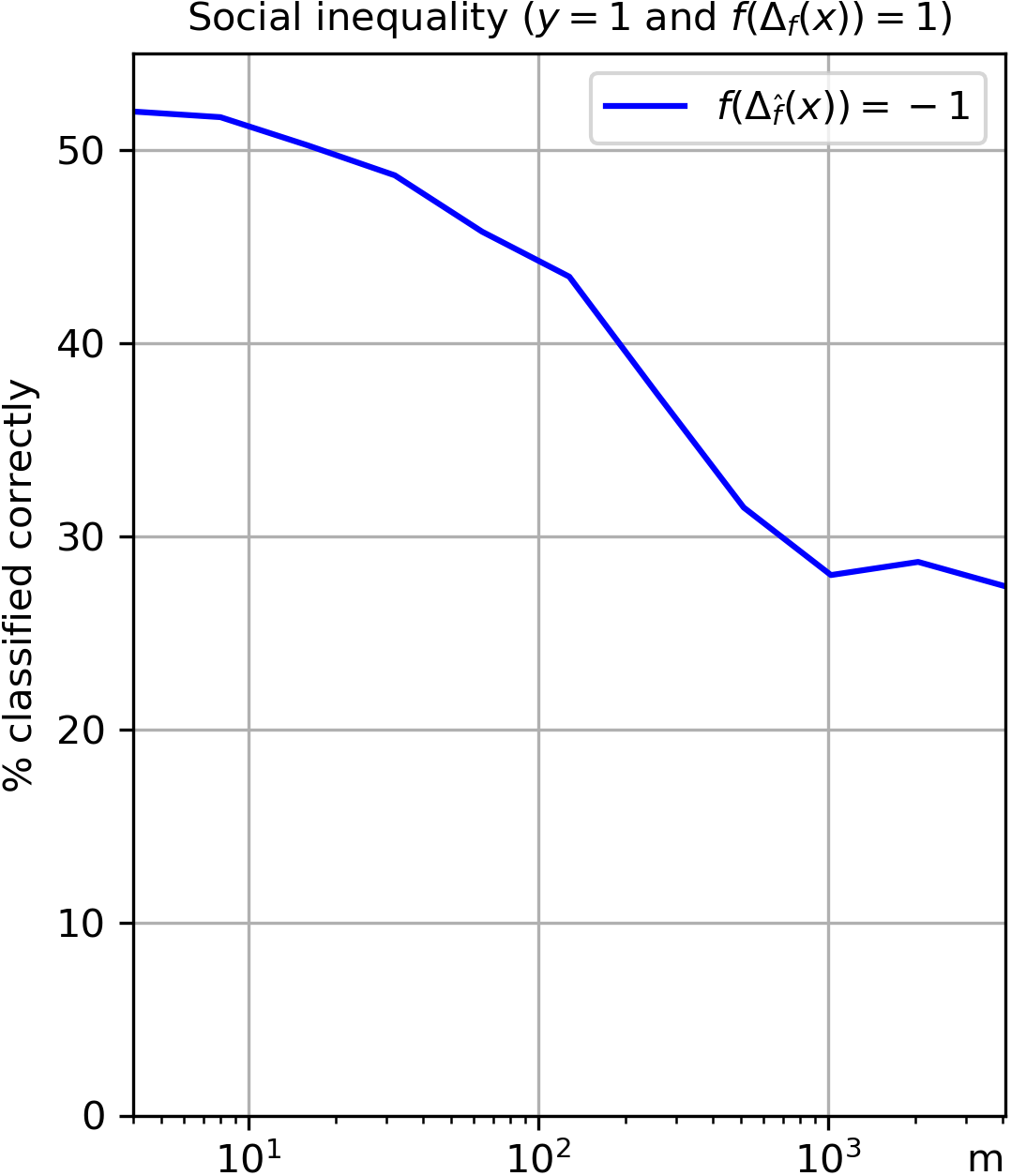}
	\caption{\textbf{(Left)} `Stories' describing outcomes
	for individual users in the Prosper.com social network.
	Users are positioned according to their features,
	with points above the line approved for loan.
	Each user is shown along with her network neighbors,
	from whom she learns $\fhat$.
	Arrows depict if and how features have been modified.
	\textbf{(Right)} Inequity in wrongful loan denial caused by an opaque policy.
	}
	\label{fig:prosper2}
\end{figure*}

\textbf{User A} is approved a loan. She is classified for approval on her true features.
She is well-connected and has both neighbors that have been denied and approved for loans,
and her estimated $\fhat$ is good. She therefore correctly infers her predicted approval
and does not modify her features.

\textbf{User B} is also approved a loan, but at an unnecessary cost.
She would have been approved on her true features,
but most of her neighbors were denied loans, leading to errors in estimating $\fhat$.
These have led her to falsely believe that she must modify her features to be approved.
The cost of modification would have been eliminated had she known $f$.

\textbf{User C} would have been approved for a loan---had she reported her true features.
However, all but one of her neighbors were denied loans.
This caused a bias in her estimate of $\fhat$, so large that modifying her
features on this account resulted in her loan being denied (unjustifiably).

\textbf{User D} is slightly below the threshold for approval. With the correct $f$,
she would have been able to modify her features to receive approval.
However, despite having neighbors who have been approved,
her estimated $\fhat$ is not exact. As a result, even after modifying her features
she remains to be denied (despite believing she will be approved).

\textbf{User E} would not have been approved. But approval is within her reach,
and had she known $f$, the cost of modification would have been small enough to make it worthwhile.
However, she is not well connected, and errors in her estimated $\fhat$
have led her to (falsely) believe that modifications for approval are too costly, and so she remains denied.


\section{All-Powerful Jury}
\label{sec:all powerful Jury}
\ganesh{In this section, we study the setting where Jury \emph{exactly knows} Contestant's response 
to his classifier $f$, despite $f$'s opacity. As mentioned in Sec.~\ref{sec:prelims}, this setting is less realistic, but we study it to show that this knowledge is precisely what Jury is missing to successfully learn a strategic classifier  when users  are in the dark. We introduce the notion of a \emph{response function} for the Contestant, denoted by $R:\calH \to \calH$. When Jury plays $f$, Contestant responds by moving the points with respect to classifier $R(f)$. For example, in the setting of Jury in the dark (Sec.~\ref{sec:jitd}), $R(f)=\fhat$ and $R$ is unknown to Jury.
In contrast, an all-powerful Jury is assumed to know the response function~$R$. Note that, when $R$ is the identity map ($R(f) = f$), the all-powerful jury setting reduces to the strategic classification setting. 

We now give an agnostic PAC learnability result for $R$-strategic classification. This generalizes a recent result of \citet{Zhang2020IncentiveAwarePL,sundaram2021paclearning}, who defined the strategic ERM rule and strategic VC dimension, using similar techniques. 
Recall 
that  $\err(f,R(f))$ denotes the error when Jury plays the  classifier $f$ and Contestant responds to $R(f)$. In this setting,  define  $R$-optimal classifier for Jury is 
$
  f^{\star}_{R} = \argmin_{f\in \mathcal{H}} \text{err}(f, R(f)). $ 
We drop the subscript $R$ from $f^{\star}_{R}$, when it is immediate from context.

First, we introduce  $R$-strategic ERM rule which minimizes the training error with an explicit knowledge of Contestants response function $R$.  Given a training set $T_J = \{x_i,h(x_i)\}_{i\in [n]}$ define 
$\widehat{\err}(f,R(f)) = 1/n \sum_{i\in [n]} \mathds{1}\{f(\Delta_{R(f)}(x)) \neq h(x) \}.$ 
The $R$-strategic ERM  returns 
$\argmin_{f\in \calH} $ $~ {\widehat{\err}(f, R(f))}.$
\begin{definition}[$R$-strategic VC dimension]\label{defintion: RSVC}
Let $\calH$ be a hypothesis class, and let $\calH' = \{ f' | \exists f\in \calH \text{ such that } f'(x) = f(\Delta_{R(f)}(x))\}$. Then $\SVC(\mathcal{H})= \VC(\calH^{'}).$
\end{definition}
%
\begin{restatable}{proposition}{thmRLearnability}
\label{prop:Learnability for all powerful Jury}
For a given hypothesis class $\calH$, Jury with sample set $T_J$ containing $n$ iid samples, using $R$-strategic ERM rules computes an $f$ such that with probability $1-\delta$ the following holds: 
$\err(f, R(f))\leq \err(f^{\star}, R(\fstar))  + \epsilon$, where  
$ \epsilon \leq \sqrt{C (d\log(d/\epsilon) + \log(\frac{1}{\delta})/n}$, $C$ is an absolute constant, and $d$ is the $R$-strategic VC dimension of $\calH$.
\end{restatable}
\begin{proof}
For each $f$, let $f'$ be  such that $f'(x) = f(\Delta_{R(f)}(x))$. It is easy to see that 
$\pr_{x\sim D}\{h(x) \neq f(\Delta_{R(f)}(x))\} =  \pr_{x\sim D}\{h(x) \neq f'(x)\}$. 
Hence, the agnostic-PAC learnability in the all powerful Jury setting can be given in terms of the VC dimension of $\calH' = \{ f' | \exists f\in \calH \text{ such that } f'(x) = f(\Delta_{R(f)}(x))\}$. It follows from Def.~\ref{defintion: RSVC} that $R$-strategic VC of $\calH$ is equal to the VC dimension of $\calH'$. The proof follows by noting that the bounds in the theorem statement are the agnostic-PAC generalization bounds for standard classification setting with the VC dimension of the hypothesis class replaced by $R$-strategic VC dimension of $\calH$. 
\end{proof}}
\section{Discussion}
The theoretical and empirical findings in this paper show that the ``optimistic firm’’ is wrong in hoping that small errors in estimating $f$ by users in the dark will remain small; as we demonstrate, small estimation errors can grow into large errors in the performance of~$f$. 
While there is a much-debated public discussion regarding the legal and ethical right of users to understand the reasoning behind algorithmic decision-making, 
in practice firms have not necessarily been propelled so far to publicize their models. Our work provides formal incentives for firms to adopt transparency as a policy.
Our experimental findings provide empirical support for these claims,
albeit in a simplified setting,
and we view the study of more elaborate user strategies
(e.g., Bayesian models, bounded-rational models,
or more nuanced risk-averse behavior)
as interesting future work.
Finally, our work highlights an interesting open direction in the study of strategic classification---to determine the optimal policy of a firm who wishes to maintain (a level of) opaqueness alongside robustness to strategic users.

\noindent \textbf{Acknowledgement:}
GG is thankful for the financial support  from the Israeli Ministry of Science and Technology grant 19400214.  VN is thankful to be supported by the European Union’s Horizon 2020 research and innovation program under grant agreement No 682203 -ERC-[Inf-Speed-Tradeoff]. 
This research was supported by the Israel Science Foundation (grant  336/18). Finally, we thank anonymous reviewers for their helpful comments.

\bibliography{refs.bib}
\bibliographystyle{icml2021}
\appendix
%
%
\section{Missing Proofs}
\label{apndx:missing proofs}
Before we proceed to the proofs, we recall the definitions of $S$, $E$, $S_{-1,1}$, $S_{1,-1}$, $E_{-1,1}$, and $E_{1,-1}$ below:
$$ S = \{x \mid f(x) \neq \fhat(x)\}, ~~ E = \{x \mid f(\Delta_f(x)) \neq  f(\Delta_{\fhat}(x))\},$$
$$ S_{-1,1} = \{x \mid f(x)=-1,\ \fhat(x)=1\},~~ S_{1,-1} = \{x \mid f(x)=1,/ \fhat(x)=-1\},$$
$$ E_{-1,1} = \{x \mid \Delta_{\fhat}(x)\in S_{-1,1},~ f(\Delta_{f}(x)) =1\},~~ E_{1,-1} = \{x \mid \Delta_{f}(x)\in S_{1,-1}\setminus \{x\},~ \Delta_{\fhat}(x)=x \}\ .$$
\subsection{Proofs of Theorems in Section \ref{sub:main}}
\mainThm*
\begin{proof}[Proof of Theorem \ref{thm:mainTheorem}]
From Eq.~\eqref{equation: first characterization of E}\ and Lemma \ref{lem: characterization of POP}, we have $\mathbb{P}_{x\sim D}\{x\in E\} = \POP + 2\POP^{-}$. Notice that points in $E^{-}$ contribute to the error of $f$, implying $\POP^{-} \leq \err(f,f)$. Recall that $\err(f,f) = \text{err}(f^{\star},f^{\star}) + \epsilon_1$. Since $\POP + 2\POP^{-} > 2(\err(f^{\star},f^{\star}) + \epsilon_1)$, we have $\POP>0$, completing the proof.
\end{proof}

\ECharacterization*
\begin{proof}
First, recall $S_{-1, 1} = \{x \mid f(x)=-1 \land \fhat(x) = 1\}$, and $S_{1, -1} = \{x \mid  f(x)=1 \land \fhat(x) = -1\}$.
Further,
\begin{align*}
E_{-1,1} &= \{x \mid \Delta_{\fhat}(x)\in S_{-1,1},~ f(\Delta_{f}(x)) =1\};\\
E_{1,-1} &= \{x \mid \Delta_{f}(x)\in S_{1,-1}\setminus \{x\},~ \Delta_{\fhat}(x)=x \}.
\end{align*}

Suppose $y = \Delta_{f}(x)$, and $z= \Delta_{\fhat}(x)$. First assuming $f(y) \neq f(z)$, we show that $x \in E$. \\

\textbf{Case a} ($y=x$ and $z\neq x$): Since $y=x$, either $f(x) = 1$ or $\{u \mid c(x,u) <2 ~\text{ and }~ f(u)=1\} = \emptyset$.
Moreover, as $z\neq x$, $\fhat(x) = -1$, and $z = \argmin_{u}\{c(x,u) \mid c(x,u)<2 ~\text{ and }~ \fhat(u)=1\}$.
We first argue in this case that $f(x)=1$. Suppose $f(x)=-1$. Then $\{u \mid c(x,u) <2 ~\text{ and }~ f(u)=1\} = \emptyset$ implying $f(z) = -1$. Since $f(x) \neq f(z)$, this gives a contradiction.
If $f(x) = 1$ and $f(z)=-1$ then as $\fhat(z)=1$, $z\in S_{-1,1}$ and hence $x\in E_{-1,1}$. \\ 

\textbf{Case b} ($y\neq x$ and $z= x$): Again as $z=x$, either $\fhat(x) = 1$ or $\{u \mid c(x,u) <2 ~\text{ and }~ \fhat(u)=1\} = \emptyset$. Also as $y\neq x$, $f(x) = -1$ and $y = \argmin_{u}\{c(x,u) \mid c(x,u)<2 ~\text{ and }~ f(u)=1\}$. If $\fhat(x) = 1$ then as $f(x) = -1$, $x\in S_{-1,1}$. Since $f(y)=1$ we have $x\in E_{-1,1}$. If $\fhat(x)=-1$ and $\{u \mid c(x,u) <2 ~\text{ and }~ \fhat(u)=1\} = \emptyset$ then $\fhat(y) =-1$ and $f(y)=1$ implying $x\in E_{1,-1}$.\\

\textbf{Case c} ($y\neq x$ and $z\neq x$): Hence, $y = \{u \mid c(x,u) <2 ~\text{ and }~ f(u)=1\}$ and $z = \argmin_{u}\{c(x,u) \mid c(x,u)<2 ~\text{ and }~ \fhat(u)=1\}$. Since $f(y)=1$, by assumption it follows that $f(z)=-1$. This implies $z\in S_{-1,1}$ and $x\in E_{-1,1}$.\\

\noindent Now we show that if $x\in E$ then $f(\Delta_{f}(x)) \neq f(\Delta_{\fhat}(x))$. Suppose $x\in E_{-1,1}$. Then $f(\Delta_{\fhat}(x)) = -1$ and $f(\Delta_{f}(x)) = 1$. Similarly if $x\in E_{1,-1}$ then $f(\Delta_{f}(x)) = 1$ and $f(\Delta_{\fhat}(x)) = -1$.
\end{proof}

\subsection{Proofs of Propositions and Corollaries in Section \ref{sub:linear}}
\textbf{Proposition \ref{prop: POP linear classifiers sufficiency condition}.} \emph{The following are sufficient for $\POP>0$:\\
(a) \,\, $\Phi_{\sigma}(t_f) -  \Phi_{\sigma}(t_f-t)  > \ell$
 ~~when $t_f > t_{\fhat}$\\
(b) \,\, $\Phi_{\sigma}(t_{\fhat}-t) -  \Phi_{\sigma}(t_f-t)  >  \ell$
~~when $t_f < t_{\fhat}$}
\begin{proof}
Recall $\ell= 2 \cdot \err(\fstar,\fstar) + 2\epsilon_1 $. We determine the enlargement set $E$ in both the cases below and then in each case use Thm.~\ref{thm:mainTheorem} to conclude the proof.

\textbf{Case a ($t_f > t_{\widehat{f}}$):} Using Theorem \ref{thm:E_partition}, $E$ can be determined easily. We note $E$ in this case in the following observation (also see Fig.  \ref{fig:E for threshold classifiers}).
\begin{observation}\label{observation: E when t_f > t_fhat}
If $t_f > t_{\widehat{f}}$ then $E = [t_{f}-t, t_f)$.
\end{observation}
From the property of normal distribution we have $\pr_{x\sim D}\{x\in E\} = \Phi_{\sigma}(t_f) -  \Phi_{\sigma}(t_f-t)$, and the result follows from Thm.~\ref{thm:mainTheorem}.\\

\textbf{Case b  ($t_{f} < t_{\widehat{f}}$): } The set $E$ in this case is as in Obs.~\ref{observation: E when t_f > t_fhat} (also see Fig. \ref{fig:E for threshold classifiers}).
\begin{observation}\label{observation: E when t_f < t_fhat}
If $t_{f} < t_{\widehat{f}}$ then $E = [t_f-t, t_{\widehat{f}} - t)$.
\end{observation}
Similar to the previous case, we have
$\pr_{x\sim D}\{x\in E\} = \Phi_{\sigma}(t_{\widehat{f}} - t) -  \Phi_{\sigma}(t_f-t)~, $
and again the proof follows from Thm.~\ref{thm:mainTheorem}.
\end{proof}

\textbf{Proposition 7.} \emph{Let $h$ be realizable and $\epsilon_2 > 2\epsilon_1$. Then, there exists $\sigma_0 >0$ such that for all $\sigma < \sigma_0$ it holds that:
\begin{equation}
\POP = \begin{cases}
 \Phi_{\sigma}(t_{\fhat} - t) - \Phi_{\sigma}(|t -t_f|) & \text{if } t_f < t_{\fhat}\\
 \Phi_{\sigma}(t_f) - \Phi_{\sigma}(|t-t_f|)
& \text{if } t_f > t_{\fhat}
\end{cases}
\end{equation}}
\begin{proof}
\begin{figure*}
\centering
    \begin{subfigure}{ }
    \includegraphics[width= 0.45 \columnwidth]{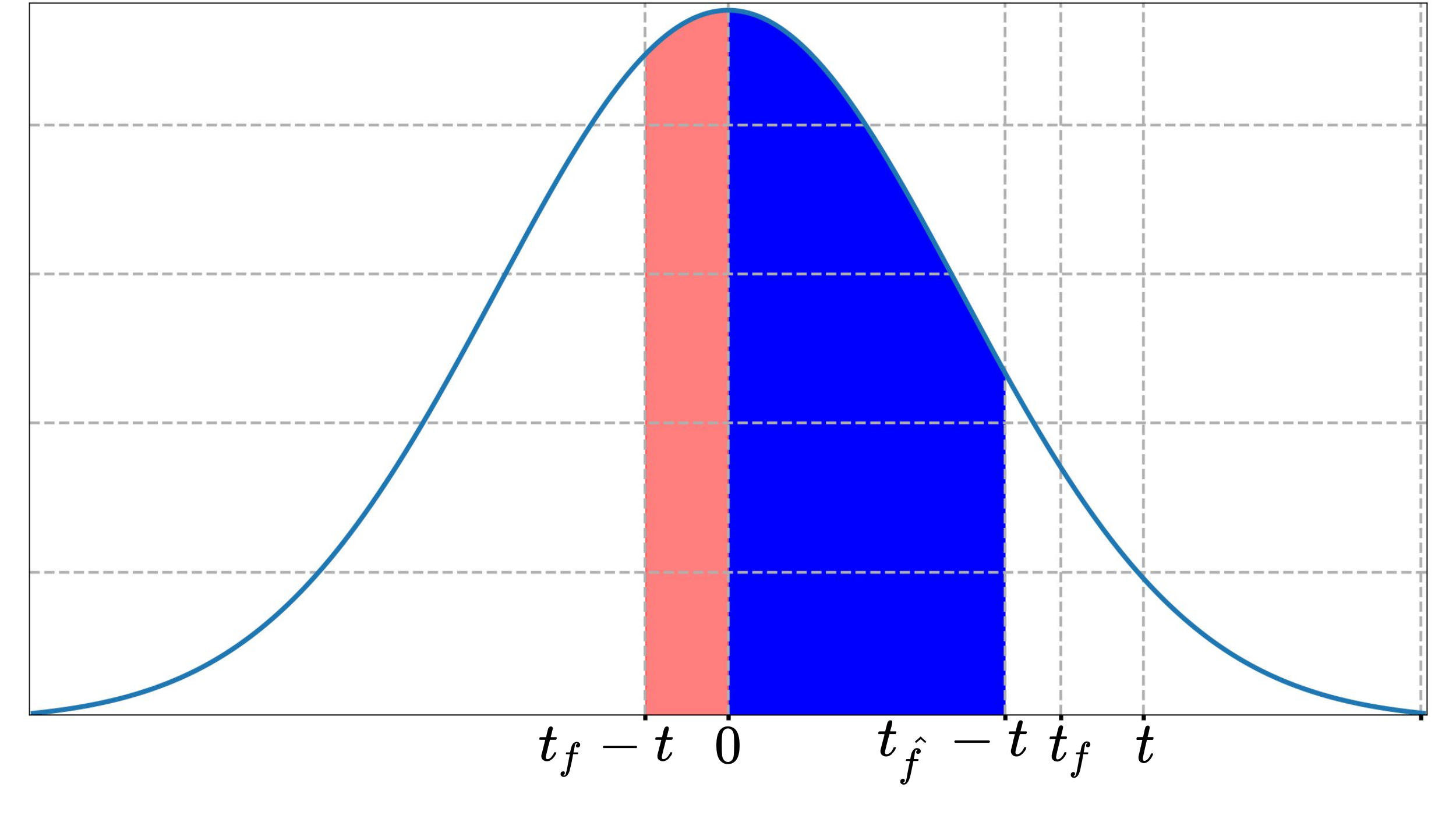}
    \end{subfigure}
        \begin{subfigure}{ }
        \includegraphics[width=0.45 \columnwidth]{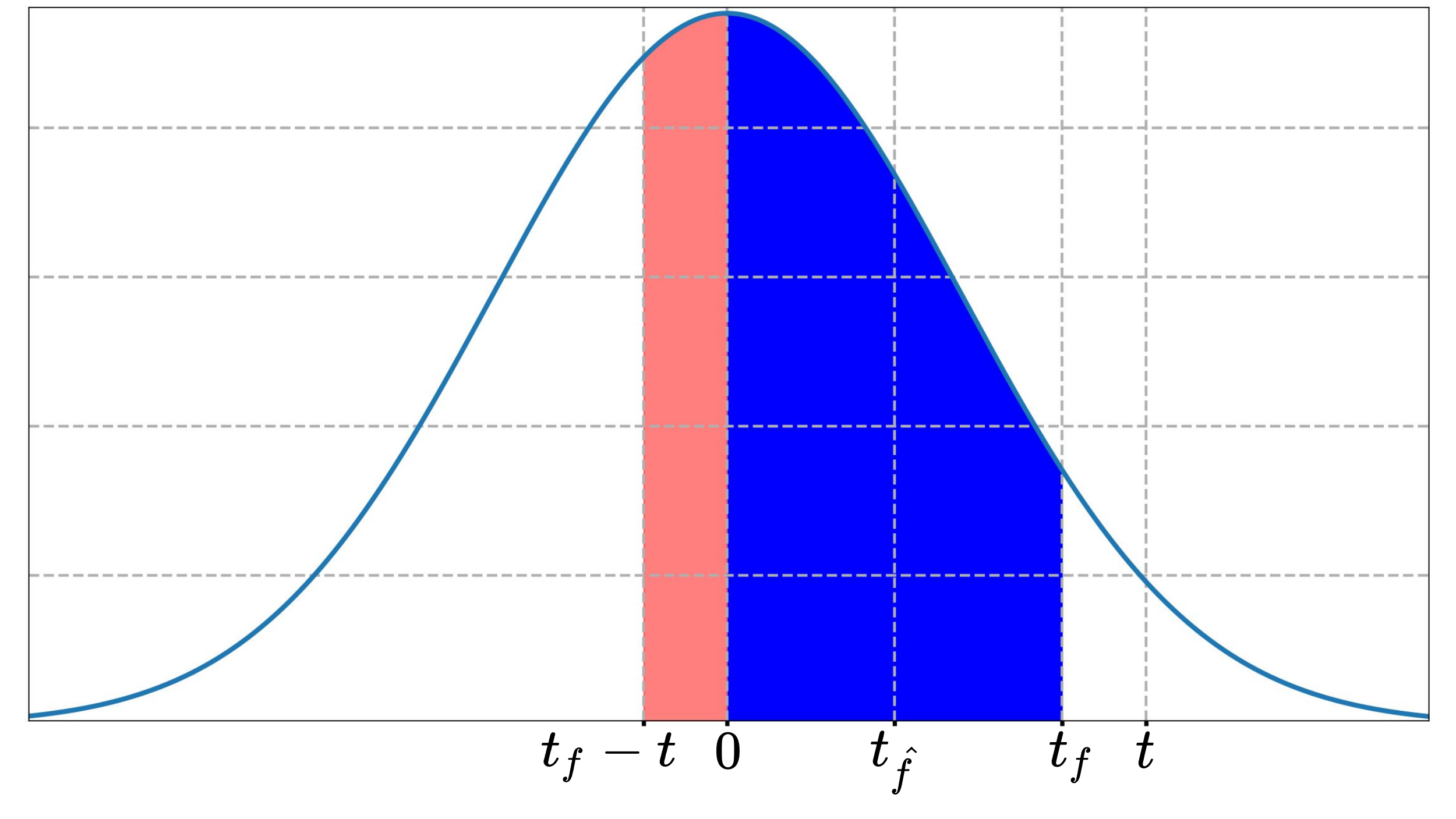}
    \end{subfigure}
    \caption{(\textbf{Left}) Set $E$ when, $t_f < t_{\fhat}$. Blue region represents $\POP^{+}$ and red region represents $\POP^{-}$. (\textbf{Right}) Set $E$ when $t_f > t_{\fhat}$. Blue color represents $\POP^{+}$ and red color represents $\POP^{-}$.  }
    \label{fig:E for threshold classifiers}
\end{figure*}
Since $h$ is realizable by a threshold classifier there is an $\alpha \in \mathbb{R}$ such that $h(x) = 1$ if and only if $x \geq \alpha$.
Without loss of generality, let $\alpha =0$ (otherwise we consider distribution $\mathcal{N}(\alpha,\sigma)$). Observe that, in this case the optimal classifier  $f^{\star}$ is defined as $\fstar(x) = 1$ for $x\geq t$ and $-1$ otherwise and satisfies  $\err(\fstar,\fstar) = 0$. Further, let $\epsilon_1 = \err(f,f)$, and $\epsilon_2 = \pr_{x\sim D}\{f(x) \neq \fhat(x)\}$. We define $\sigma_0 >0$ such that $\Phi_{\sigma_0}(\frac{t}{2}) = \frac{1}{2} + \epsilon_1$.


We show in the lemma below that $t_f > \frac{t}{2}$ for the chosen values of $\sigma$.
\begin{lemma}
\label{lem: t_f > t/2}
For all $0<\sigma < \sigma_{0}$, $t_f > \frac{t}{2}>0.$
\end{lemma}
\begin{proof}
Suppose for contradiction $t_f \leq \frac{t}{2}$. Then $t- t_{f} \geq \frac{t}{2}$.   
 This implies
\begin{align*}
    \epsilon_1 &:=  \err(f,f) =  \frac{1}{2} - \Phi_{\sigma}(t_f - t) \\  & \underset{(i)}{=}   \Phi_{\sigma}(t-t_f) - \frac{1}{2}
                \underset{(ii)}{\geq}   \Phi_{\sigma}(\frac{t}{2}) - \frac{1}{2} \underset{(iii)}{>} \Phi_{\sigma_0}(\frac{t}{2}) - \frac{1}{2}  =\epsilon_1
\end{align*}
In the above, $(i)$ follows from the symmetry property of Normal distribution, $(ii)$ is an immediate consequence of monotonicity property of CDF function $\Phi(.)$ and finally $(iii)$ follows from the choice of $\sigma$.  This completes the proof.
\end{proof}
Hence, we assume throughout the proof $t_f > \frac{t}{2}> 0$. Now we determine \POP\ and show that it greater than $0$ for the two cases.

\textbf{Case 1  ($t_{f} < t_{\widehat{f}}$): } The enlargement set $E = [t_f-t, t_{\widehat{f}} - t]$ (determined in the proof of Prop. \ref{prop: POP linear classifiers sufficiency condition}, see Obs.~\ref{observation: E when t_f < t_fhat}). Now in Obs.~\ref{observation: t < t_fhat}, we show that $t <t_{\widehat{f}}$ in this case.
\begin{observation}\label{observation: t < t_fhat}
If $t_{f} < t_{\widehat{f}}$ then $ t < t_{\widehat{f}} $.
\end{observation}
\begin{proof}
Note that $\epsilon_1 = \frac{1}{2} - \Phi_{\sigma}(t_{f}-t)$. Further, from definition $\epsilon_2 = \Phi_{\sigma}(t_{\widehat{f}}) - \Phi_{\sigma}(t_f)$. Suppose $t \geq t_{\widehat{f}}$. Then as $t_{f}>0$
\begin{equation*}
  \Phi_{\sigma}(t_{\widehat{f}}) - \Phi_{\sigma}(t_f)~ <~ \Phi_{\sigma}(t-t_f) - \Phi_{\sigma}(0)
  ~=~  \epsilon_1
\end{equation*}
If $\epsilon_2<\epsilon_1< 2\epsilon_1$ then this contradicts the assumption that $\epsilon_2 \geq 2\epsilon_1$.
\end{proof}
Hence,  either $t_{f} < t < t_{\widehat{f}}$ or $t< t_f < t_{\widehat{f}}$. First, we analyse the case when $t_{f} < t < t_{\widehat{f}}$. Recall that the price of opacity is defined as  $\POP = \err(f, \widehat{f}) - \err(f,f)$. Observe that for all $x\in [0,t_{\widehat{f}} - t)$, $h(x) = f(\Delta_f(x))$  and for all $x\in [t_f-t,0)$, $h(x) \neq f(\Delta_f(x))$. In particular, the points in $[0,t_{\widehat{f}} - t)$ contribute positively to the price of opacity, and the points in $[t_f-t,0)$ contribute negatively to the price of opacity. Hence,
\begin{align*}
\POP &= (\Phi_{\sigma}(t_{\widehat{f}}- t) - \Phi_{\sigma}(0)) - (\Phi_{\sigma}(0) - \Phi_{\sigma}(t_f-t))\\
   &= (\Phi_{\sigma}(t_{\fhat} - t )  - \Phi_{\sigma}(0)) - (\Phi_{\sigma}(t - t_{f}) - \Phi_{\sigma}(0))\\
    & = \Phi_{\sigma}(t_{\fhat} - t) - \Phi_{\sigma}(t -t_f)
\end{align*}
Finally, in this case $\POP$ is at least $0$ follows from the following observation.
\begin{observation}
$t_{\fhat} - t > t - t_f$
\end{observation}
\begin{proof}
By assumption we have $\epsilon_2 \geq 2\epsilon_1$. This implies that
\begin{align*}
    \Phi_{\sigma}(t_{\fhat}) & \geq \Phi_{\sigma}(t_f) + 2 (\Phi_{\sigma}(t - t_f) -1/2)\\
    & > \Phi_{\sigma}(t_f) + \Phi_{\sigma}(2t - t_f) - \Phi_{\sigma}(t_f) \\
    & = \Phi_{\sigma}(2t - t_f)
\end{align*}The second line in the above equation follows for any $\sigma > 0$ as $0< t_f < t < t_{\fhat}$. This completes the proof.
\end{proof}
Similarly, when $t< t_f < t_{\widehat{f}}$, it can be shown that $E =$  $[t_f-t,$ $t_{\widehat{f}}-t)$, and $\POP = \Phi_{\sigma}(t_{\widehat{f}}-t) - \Phi_{\sigma}(t_{f}-t) > 0$.\\

\textbf{Case 2 ($t_f > t_{\widehat{f}}$):} The set $E = [t_{f}-t, t_f)$ in this case (see Observation \ref{observation: E when t_f > t_fhat}). If $t_f > t$ then it is easy to show that $\POP = \Phi_{\sigma}(t_f) - \Phi_{\sigma}(t_f-t) > 0$. Next we analyse the price of opacity for $t_f<t$.
Observe that for all $x\in [0,t_f)$ $h(x) = f(\Delta_f(x))=1$, and for all $x\in [t_{f}-t,0)$, $h(x) = -1$ and $f(\Delta_f(x))=1$. In particular the points in $[0,t_f)$ contribute positively to the price of opacity, whereas the points in $[t_{f}-t,0)$ contribute negatively to the price of opacity (see Fig.~\ref{fig:E for threshold classifiers}). Hence,
\begin{align*}
\POP &= (\Phi_{\sigma}(t_f) - \Phi_{\sigma}(0)) - (\Phi_{\sigma}(0) - \Phi_{\sigma}(t_f-t))\\
&= (\Phi_{\sigma}(t_f) - \Phi_{\sigma}(0)) - (\Phi_{\sigma}(t-t_f) - \Phi_{\sigma}(0))\\
&= \Phi_{\sigma}(t_f) - \Phi_{\sigma}(t-t_f)~.
\end{align*}
Since $t_f >\frac{t}{2}$ (from Lem.~\ref{lem: t_f > t/2}), $t_f > t-t_f$, and hence, $\POP>0$.
\end{proof}
\textbf{Corollary 6.} \emph{Suppose $h$ is realizable, $t_f < t$, and $\epsilon_2 > 2\epsilon_1$. Then there exists $\sigma_0 >0$ such that for all $\sigma < \sigma_0$, $\POP > 0$ if and only if $\pr_{x\sim D}\{E\} > 2\epsilon_1$.}
\begin{proof}
Similar to Prop. \ref{prop: POP necessary and sufficient}, without loss of generality we have $h(x) = 1$ if and only if $x \geq \alpha$.
Hence, the optimal classifier  $f^{\star}$ is defined as $\fstar(x) = 1$ for $x\geq t$ and $-1$ otherwise and satisfies  $\err(\fstar,\fstar) = 0$.
It follows from Thm.~\ref{thm:mainTheorem} that if $\pr_{x\sim D}\{E\} > 2\epsilon_1$ then $\POP > 0$. Now we prove the other direction.

In Props. \ref{prop: POP linear classifiers sufficiency condition} and \ref{prop: POP necessary and sufficient} we showed that if $t_f> t_{\fhat}$ then
$$\pr_{x\sim D}\{x\in E \}  = \Phi_{\sigma}(t_f) -  \Phi_{\sigma}(t_f-t)$$
$$\POP = \Phi_{\sigma}(t_f) - \Phi_{\sigma}(t-t_f)\ .$$
Similarly, we also showed that if $t_f <  t_{\widehat{f}}$
$$\pr_{x\sim D}\{x\in E \}  = \Phi_{\sigma}(t_{\fhat}-t) -  \Phi_{\sigma}(t_f-t)$$
$$\POP = \Phi_{\sigma}(t_{\fhat}-t) - \Phi_{\sigma}(t-t_f)\ .$$
From the two equations in each case we conclude the following holds in both the cases:
$$ \POP = \pr_{x\sim D}\{x\in E\} + \Phi_{\sigma}(t_f-t) - \Phi_{\sigma}(t-t_{f})\ .$$
Also, note that $\epsilon_1 = 1/2 -\Phi_{\sigma}(t_f-t)$. If $\POP > 0$ then
\begin{align*}
    \pr_{x\sim D}\{x\in E\} + \Phi_{\sigma}(t_f-t) - \Phi_{\sigma}(t-t_{f}) &> 0 \\
    \pr_{x\sim D}\{x\in E\} &>  \Phi_{\sigma}(t-t_{f}) - \Phi_{\sigma}(t_f-t) \\
    &> 1 - 2\Phi_{\sigma}(t_f-t) = 2\epsilon_1
\end{align*}
The first inequality in the last line of the above equation follows from the property of normal distribution centered at $0$.
\end{proof}

\textbf{Corollary \ref{cor: POP  linear classifiers}.} \emph{Suppose $ \epsilon_2 >  2\sqrt{ \log \frac{4}{\delta} / n }$, where $\delta \in (0,1)$. Then  with probability at least $1-\delta$, $\POP$ is as in Eq.~\eqref{eqn: POP characterization equation}.}
\begin{proof}
As the strategic VC dimension of the threshold classifier for an admissible cost function is at most $2$, from the strategic-PAC generalization error bound we have with probability at most $\delta$ (over the samples in $T_J$), $\err(f,f)  = \epsilon_1 \leq \sqrt{\frac{\log \frac{8}{\delta}}{n}}$. Using this the corollary follows from Prop. 7.
\end{proof}

\section{Experiments} \label{app:experiments}
\label{apndx:experiments}

\subsection{Synthetic experiments: additional results}

\subsubsection{Multi-Variate Normal Distribution}
Our theoretical results in Sec. \ref{sub:linear} consider a 1D Normal distribution.
In this section we empirically study \POP\ under a split multi-variate Normal
distribution with means $\mu=$ and diagonal covariance matrix
$\Sigma = 2I$.
We set the cost to be linear-separable with weight vector $\alpha=(1,0,\dots,0)$.
We measure pop for increasing dimensionality $d$.
Across $d$ we fix $\epsilon_d$ by choosing $m(d)$
for which the error is at most 0.01 w.p. at least 0.95 over
sample sets (i.e., as in PAC),
and to simplify,
mimic a setting with large $n$ by fixing $f=f^*$
(which thresholds along the $x_1$ axis at 2).
Results for were averaged over 50 random draws of $n=100$ test points
moved by a learned $\fhat$.

Figure \ref{fig:negpop_mvn} (left, center) shows results for increasing $d$.
As can be seen \POP\ remains quite large even for large $d$.
Interestingly, \POP\ begins lower for $d=1$ (the canonical case),
then increases abruptly at $d=2$,
only to then slowly decrease until plateauing
at around $d=150$.
Interestingly, the low \POP\ at $d=1$ is due to high variance:
roughly half of the time $\POP \approx 0.35$,
whereas the other half $\POP=0$. This behavior virtually diminishes for $d>1$.

\subsubsection{Negative POP}
As noted in Sec. \ref{sub:POP}, \POP\ can also be negative.
Fig. \ref{fig:negpop_mvn} (right) presents such a case,
in which we used a mixture of two 1D Normals
centered at 0 and 1 and with diagonal unit covariances.
Here, as $m$ increases, $\fhat$ better estimates $f$,
but \POP\ remains significantly negative throughout.

\begin{figure}[t!]
	\centering
	\includegraphics[width=0.34\columnwidth]{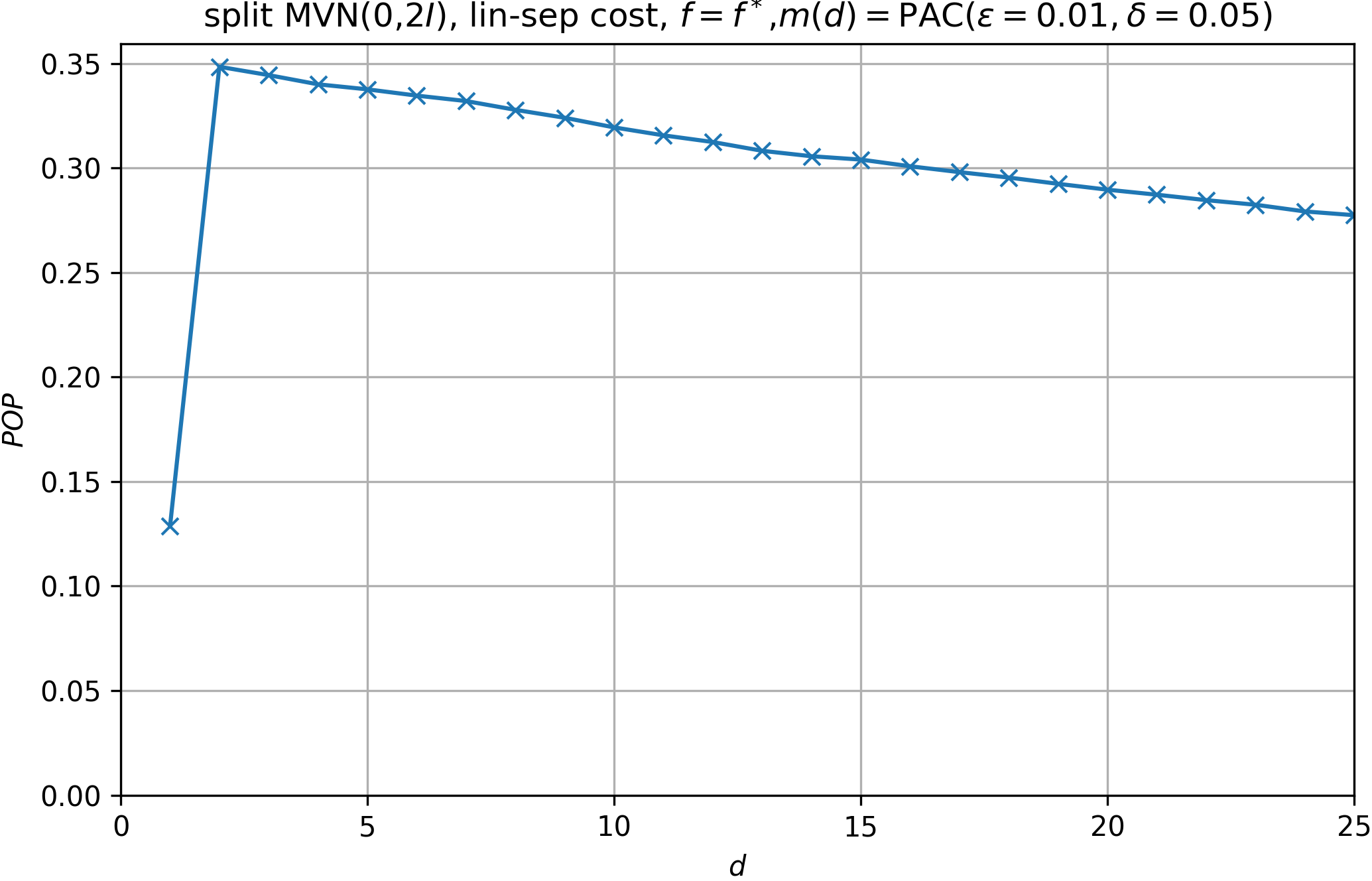}
	\includegraphics[width=0.34\columnwidth]{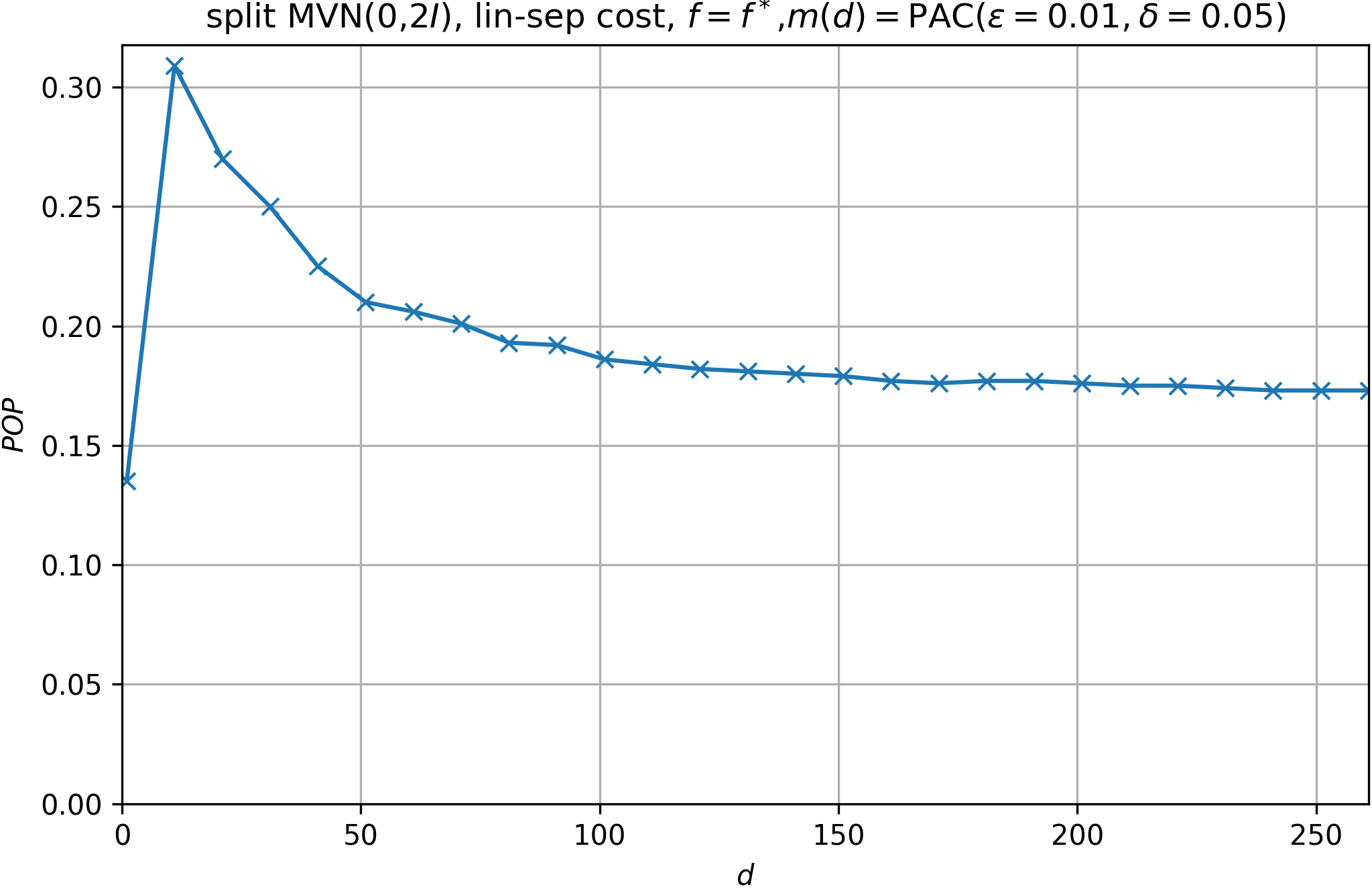}\,
	\includegraphics[width=0.28\columnwidth, height=3cm]{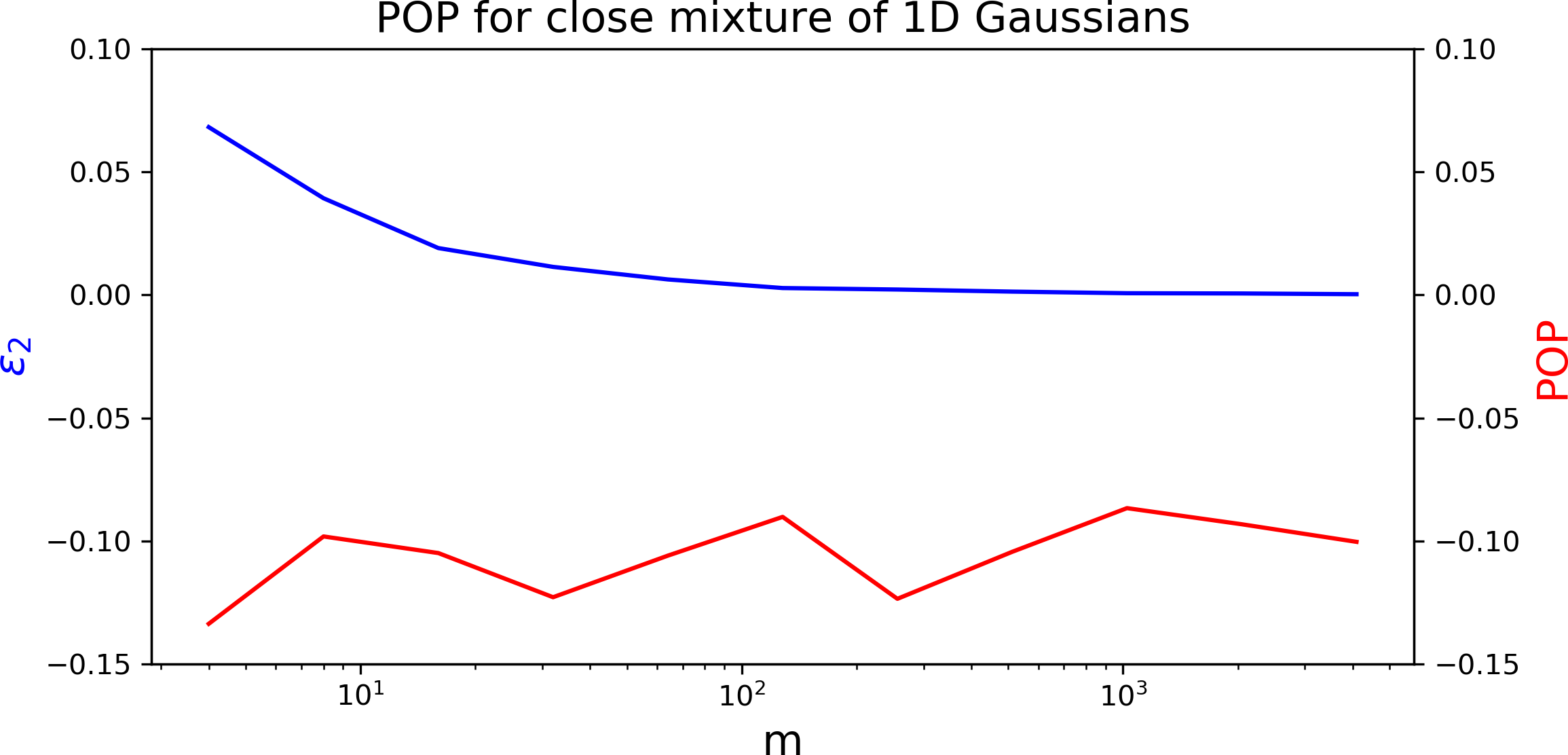}
	\caption{
	(\textbf{Left, center})
	\POP\ for split MVNs
	of increasing dimension $d$ (center: high-resolution for small $d$;
	right: low-resolution for large $d$).
	(\textbf{Right}) Mixture of 1D Normals giving strictly negative \POP.
	}
	\label{fig:negpop_mvn}
\end{figure}



\subsection{Loans experiment: details}

\subsubsection{Data}
Our experiments makes use of data cross-referenced from
two distinct datasets related to the Prosper.com peer-to-peer loans platform.
The first dataset\footnote{\url{https://www.kaggle.com/yousuf28/prosper-loan}} includes $n=113,937$ examples
having $d=81$ features (not all useful for classification purposes; see below).
Examples include only loan requests that have been approved,
i.e., that sufficient funds were allocated by borrowers on the platform
to support the request.
This dataset has informative features that are useful for predictive purposes,
but does not include any information on social connections.
To complement this,
we use a second dataset\footnote{Provided to us upon request by the authors of
\citet{kumarnetwork}.}, which includes a full data-dump
of all data available on the platform, including: loan request (both approved and disapproved),
bid history, user profiles, social groups, and social connections.\footnote{The data dump is dated to 2009, shortly before Prosper.com underwent significant changes.
These included changes in
how loan requests are processed and attended,
how credit risk is evaluated by the platform,
and how users can interact.
As part of these changes,
Prosper.com discontinued public access to its API
for obtaining data.}
This dataset is richer in terms of available data types,
but is missing several highly informative features that
appear in the first dataset.
Hence, we take the intersection of both datasets
(entries were matched based on user-identifying member keys).
Our merged dataset includes $n=20,222$ examples of loan requests
that are described by informative features and whose
corresponding users
can be linked to others within the social network.


\paragraph{Labels.}
Labels in our data are determined based on a loan request's
\emph{credit grade} (\cg),
a feature indicating the risk level of the request
as forecasted by the platform (for which they use proprietary
tools and information that is not made public).
Values of \cg\ are in the range AA,A,B,C,D,E,HR, with each level
corresponding to a risk value indicating the likelihood of default.\footnote{\url{https://www.prosper.com/invest/how-to-invest/prosper-ratings/?mod=article_inline}}
Because \cg\ is essentially a prediction of loan default made by
the platform, we use it as labels, which we binarize via $y=\1{\text{\cg} \ge B}$ to obtain a roughly balanced dataset.
In this way, learning $f$ mimics the predictive process
undertaken by the platform and for the same purposes
(albeit with access to more data).

\paragraph{Features.}
Although the data includes many features, most of them are redundant
for predictive purposes, others contain many missing entries,
and yet others are linked directly to $y$ in ways that make them
inappropriate to use as input (e.g., borrower rate, which is fixed by the platform using a one-to-one mapping from \cg, makes prediction trivial).
Due to this,
and because we also study cases in which
the number of user-held examples $m$ is small,
we choose to work with a subset of six features,
chosen through feature selection for informativeness and relevance.
available credit, amount to loan, percent trades not delinquent, bank card utilization, total number of credit inquiries, credit history length.
These remain sufficiently informative of $y$:
as noted, a non-strategic linear baseline (on non-strategic data)
achieves 84\% accuracy.
Fig.~\ref{fig:prosper1} (inlay) shows that errors in
the estimation of $\fhat$ (i.e., $\epsilon_2$) are manageable even for small $m$.

\paragraph{Social network.}
In its earlier days, the Prosper.com platform included
a social network that allowed users to form social ties
beyond those derived from financial actions.
In our final experiment in Sec. \ref{sec:experiments},
we use the social network to determine the set of examples available
for each user $x$ for training $\fhat_x$,
and in particular, include in the sample set all loan requests
of any user that is at most at distance two from $x$
(i.e., a friend or a friend-of-a-friend).
The social network is in itself quite large;
however, since we have informative features only for loan requests
that have been approved (see above),
in effect we can only make use of a smaller subset of the network
which includes 994 users
that have at least one social connection.



\subsubsection{Experimental details}

\paragraph{Preprocessing.}
All features were standardized (i.e., scaled to have zero mean
and unit standard deviation).

\paragraph{Training.}
The data was split 85-15 into train and test sets,
respectively.
The train set was further partitioned for tuning the amount of regularization, but optimal regularization proved to be negligible,
and final training was performed on the full train set.
Evaluation was done on points in the held-out test set,
and for settings in which $\fhat$ is used,
training sets for $\fhat$ were sampled from the train set
and labeled by $f$.
As noted, rejection sampling was used to ensure that
sample sets include at least one example from each class.

\paragraph{Learning algorithms.}
For the non-strategic baseline we used the scikit-learn
implementation of SVM.
For the strategic learning algorithm we used our own implementation
of the algorithm of \citet{hardt2016strategic}.

\paragraph{Computational infrastructure.}
All experiments were run on a single laptop
(Intel(R) Core(TM) i7-9850H CPU 2.6GHz, 2592 Mhz, 6 core, 15.8GB RAM).

\paragraph{Code.}
Code is publicly available at
\url{https://github.com/staretgicclfdark/strategic_rep}



\subsubsection{Visualization}
For visualization purposes, we embed points $x\in\R^6$ in $\R^2$.
Our choice of embedding follows from the need to depict
relations between the position of points in embedded space
and how they are classified, i.e., their value under $f(x)$.
To achieve this, we embed by partitioning features into two
groups, $A_1,A_2 \subset \{1,\dots,6\}$, and decompose
$f(x)=f_1(x)+f_2(x)$ where $f_1$ operates
only on the subset of features in $A_1$ and $f_2$
on features in $A_2$. We then set the embedding $\rho$ to be $\rho(x)=(f_1(x),f_2(x))$, i.e.,
values for the x-axis correspond to $f_1$,
and values for the y-axis correspond to $f_2$.
In this way, embedded points lie above the line $f(x)=0$
iff they are classified as positive
(note that this holds for any partitioning of features).
However, such a projection does not preserve
distances or directions, and so the distance of points
from $f$ and the direction and distance of point movement
do not necessary faithfully represent those in the original feature space.


\begin{figure}[t!]
	\centering
	\includegraphics[width=0.45\columnwidth]{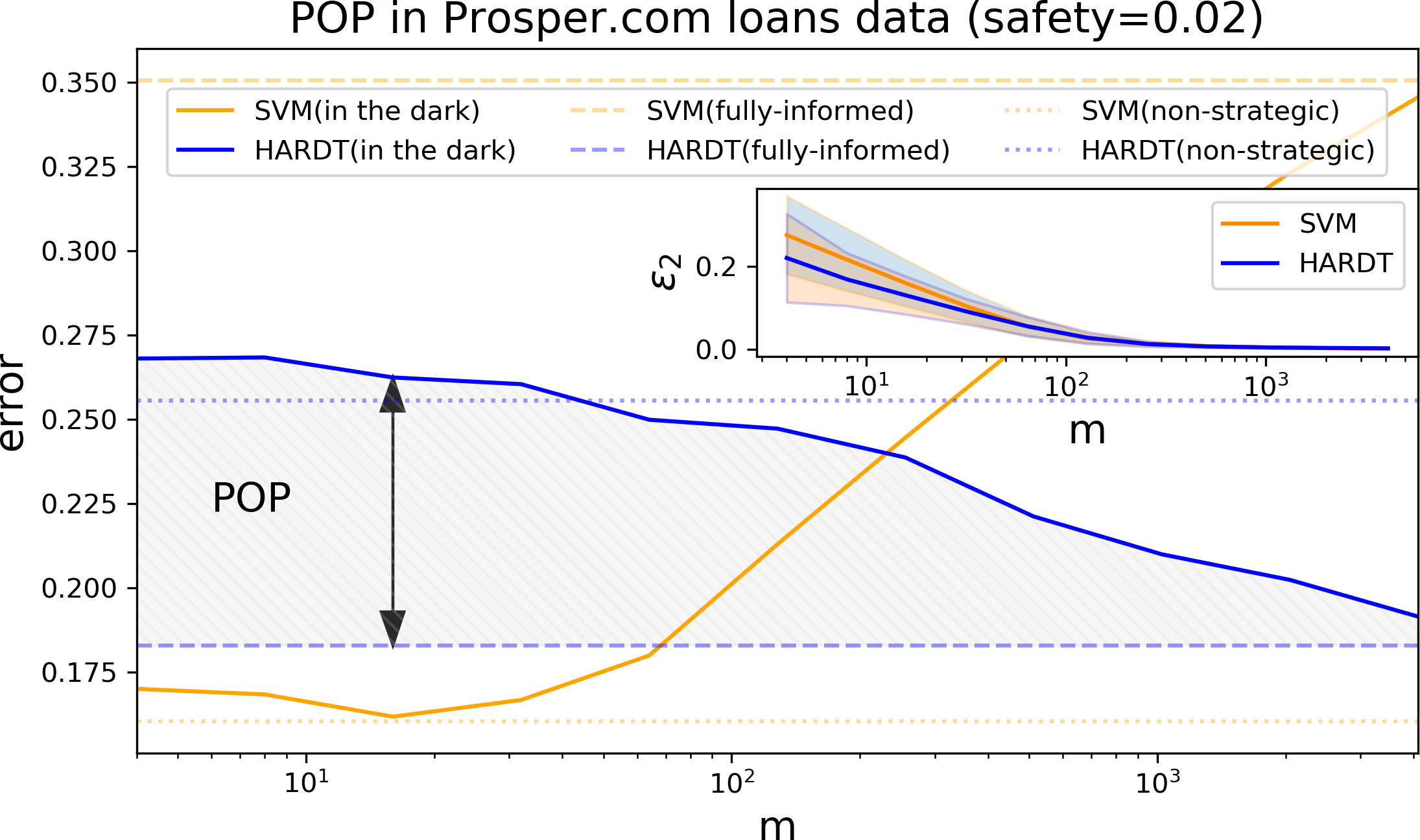}
	\qquad
	\includegraphics[width=0.45\columnwidth]{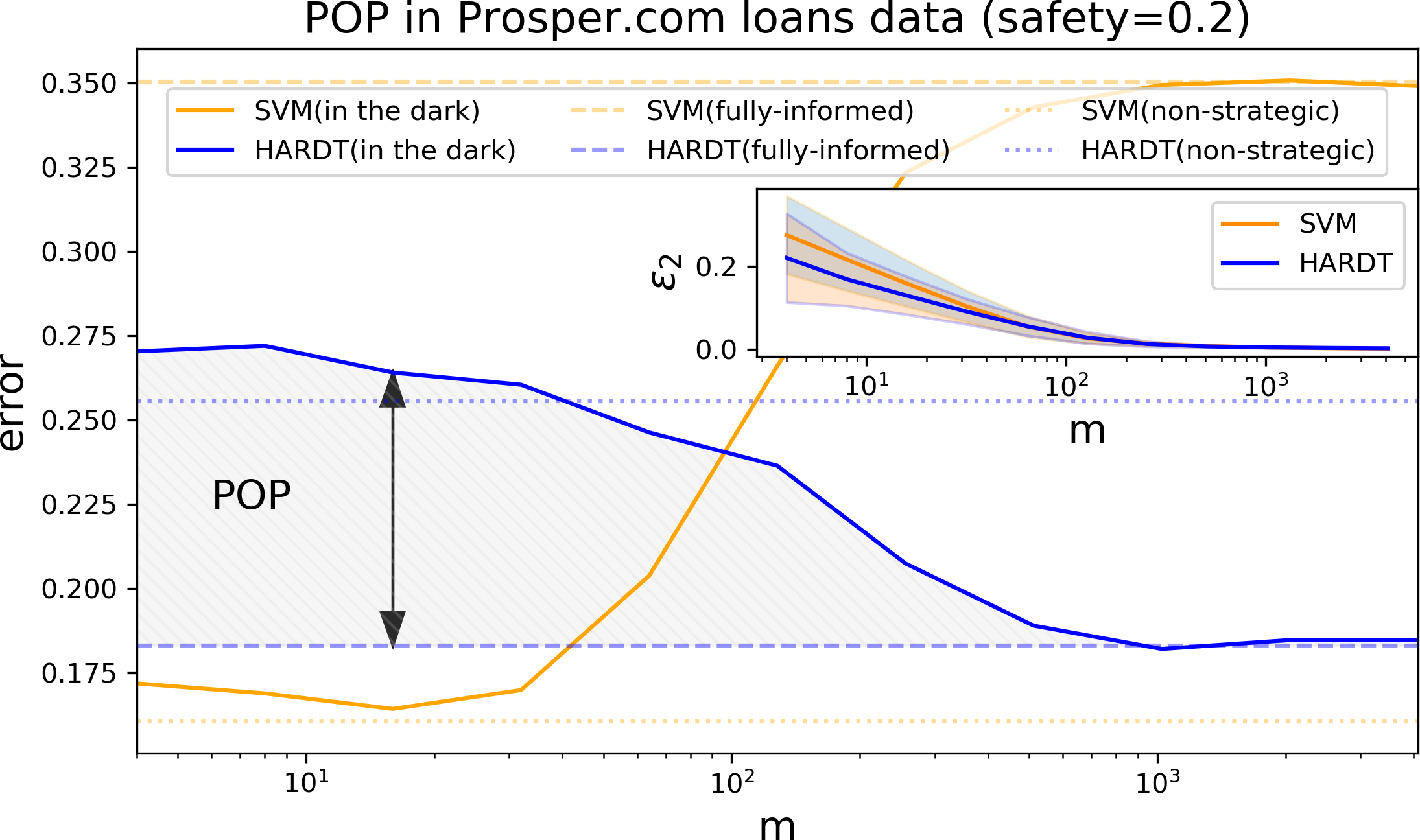}
	\caption{\POP\ and estimation error for $\fhat$ ($\epsilon_2$)
	in the safe Contestant setting,
	for safety budgets $k=0.02$ (1\% of the total budget) and $k=0.2$ (10\% of the total budget).
	}
	\label{fig:safety_pop}
\end{figure}

\subsection{Loans experiment: additional results}\label{secappendix: loan experiment addtional result}
Our main results relate to a Contestant playing $\fhat$ exactly.
Here we empirically study a variation on this model.
If Contestant knows that $\fhat$ only approximates $f$,
she may be willing to make up for possible errors in estimation
by incurring an additional cost.
We refer to such a Contestant as \emph{safe},
and model her behavior as follows.
First, Contestant computes her best-response w.r.t. $\fhat$,
denoted $x'$.
This provides a direction of movement $r = x'-x$.
Recalling that the distance of movement was set to minimize the cost,
the safe Contestant then invests an additional
$k$ units of cost
(at most, without exceeding her overall cost budget of 2)
to move further in the same direction $r$.

Figure \ref{fig:safety_pop} show \POP\ behavior for the safe Contestant setting
for `safety' budgets $k=0.02$ (1\% of the total budget) and $k=0.2$ (10\% of the total budget). As can be seen, the main trends in the results match those
in Sec. \ref{sec:exp_loans} (Fig. \ref{fig:prosper1}).
Notice that for large $m$, \POP\ is now smaller,
as the increased distances can now move points beyond
the region of disagreement.
However, closing the \POP\ gap requires many samples ($m \approx 1,000$)
and a large safety budget (10\% of total).
Figure \ref{fig:safety_social_ineq} shows social inequality measures
for the safe Contestant setting,
with matching results indicating that large $m$ and $k$
are necessary to maintain social equality.

\begin{figure}[t!]
	\centering
	\includegraphics[width=0.3\columnwidth]{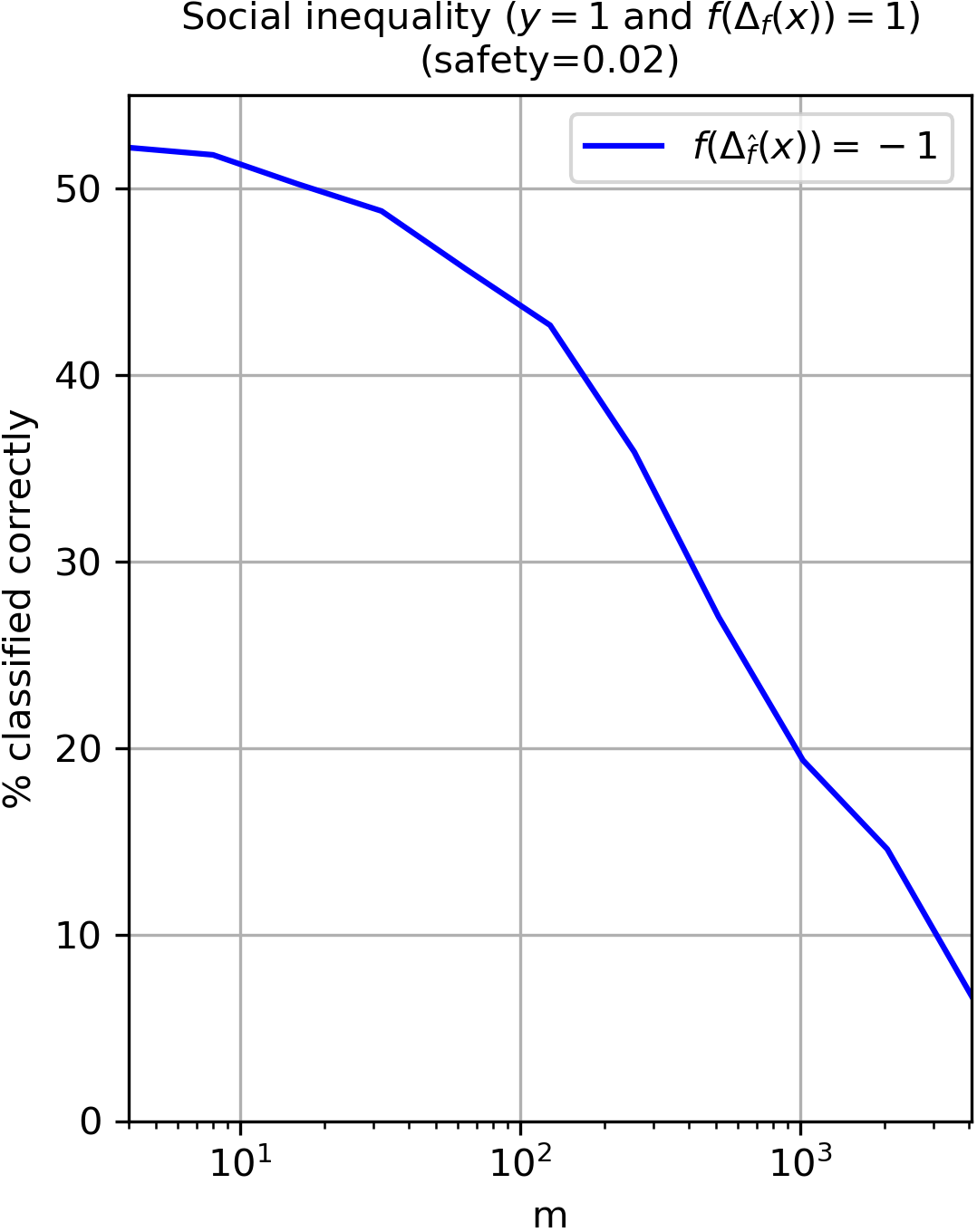}
	\qquad \qquad \qquad
	\includegraphics[width=0.3\columnwidth]{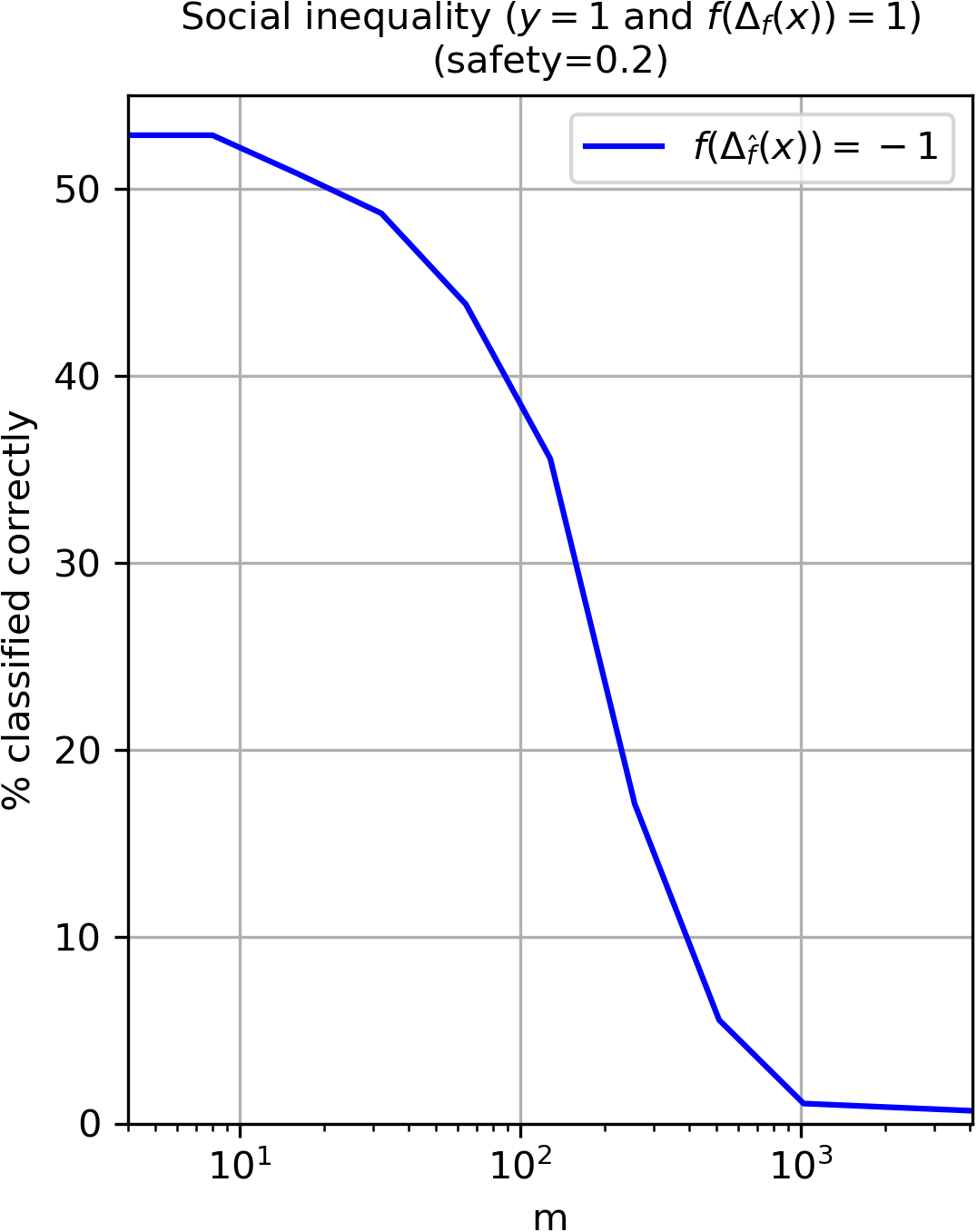}
	\caption{Social inequality for $\fhat$ ($\epsilon_2$)
	in the safe Contestant setting,
	for safety budgets $k=0.02$ (1\% of the total budget) and $k=0.2$ (10\% of the total budget).}
	\label{fig:safety_social_ineq}
\end{figure}



\end{document}